%% file: neurips_2026.tex
\documentclass{article}

\usepackage[preprint]{neurips_2026}

\usepackage[utf8]{inputenc} % allow utf-8 input
\usepackage[T1]{fontenc}    % use 8-bit T1 fonts
\usepackage{hyperref}       % hyperlinks
\usepackage{url}            % simple URL typesetting
\usepackage{booktabs}       % professional-quality tables
\usepackage{amsfonts}       % blackboard math symbols
\usepackage{nicefrac}       % compact symbols for 1/2, etc.
\usepackage{microtype}      % microtypography
\usepackage{amsmath}
\usepackage{amssymb}
\usepackage{mathtools}
\usepackage{amsthm}
\usepackage{algorithm}      
\usepackage{algorithmicx}  
\usepackage{algpseudocode}  
\usepackage{hyperref}
\usepackage[most]{tcolorbox}
\usepackage[table]{xcolor}
\usepackage{xcolor}
\usepackage{colortbl}
\hypersetup{
  colorlinks=true,
  citecolor={blue!50!black},
  linkcolor={red!50!black},
  urlcolor={blue!60!black}
}
\theoremstyle{plain}
\newtheorem{theorem}{Theorem}[section]
\newtheorem{remark}{Remark}[section]
\newtheorem{proposition}{Proposition}[section]
\newtheorem{corollary}{Corollary}[section]
\theoremstyle{definition}

\definecolor{UCPO}{RGB}{31,119,230}
\title{Uniform-Correct Policy Optimization: Breaking RLVR's Indifference to Diversity}

\author{ 
    \textbf{Anamika Lochab},  \ 
    \textbf{Bolian Li},  \ 
     \textbf{Ruqi Zhang}\\
    Department of Computer Science\\Purdue University, West Lafayette\\ 
    \texttt{\{alochab,  li4468, ruqiz\}@purdue.edu} }

\begin{document}

\maketitle

\begin{abstract}
Reinforcement Learning with Verifiable Rewards (RLVR) has achieved substantial gains in single-attempt accuracy (Pass@1) on reasoning tasks, yet often suffers from reduced multi-sample coverage (Pass@K), indicating diversity collapse. We identify a structural cause for this degradation: common RLVR objectives, such as GRPO, are indifferent to how probability mass is distributed among correct solutions. Combined with stochastic training dynamics, this indifference induces a self-reinforcing collapse, in which probability mass concentrates on a narrow subset of correct outputs while alternative valid solutions are suppressed. We formalize this collapse mechanism and further characterize the optimal policy structure under two complementary criteria: robustness and entropy-regularized optimality, which identify the Uniform-Correct Policy as uniquely optimal.
Motivated by this analysis, we propose  \emph{Uniform-Correct Policy Optimization} (UCPO), a modification to GRPO that adds a conditional uniformity penalty on the policy’s distribution over correct solutions. The penalty redistributes gradient signal toward underrepresented correct responses, encouraging uniform allocation of probability mass within the correct set. Across three models (1.5B–7B parameters) and five mathematical reasoning benchmarks, UCPO improves Pass@K and diversity while maintaining competitive Pass@1, achieving up to +10\% absolute improvement on AIME24 at Pass@64 and up to 45\% higher equation-level diversity within the correct set.
\end{abstract}

\section{Introduction}

\input{sections/introduction}

\section{Related Works}

\input{sections/new_related_works}

\section{Preliminaries}
\input{sections/preliminiaries}

\section{Structural Causes of Diversity Collapse in RLVR} 
\input{sections/struc_collapse_of_div}

\section{The Optimal Policy}
\input{sections/optimal_policy}

\section{UCPO: Uniform Correct Policy Optimization}
\input{sections/ucpo}

\section{Experiments}
\input{sections/experiments}

\section{Conclusion}

We identify a structural limitation of RLVR training: while it increases the probability of correct solutions, it is indifferent to how probability mass is distributed within the correct set. This indifference, coupled with on-policy sampling dynamics, leads to diversity collapse and degraded multi-sample coverage (Pass@K). We show that the optimal policy is the Uniform-Correct Policy, which distributes probability evenly across correct solutions. 
To achieve this structure, we proposed UCPO, to reweight gradients conditional on correctness, amplifying underrepresented correct solutions while tempering dominant ones. Across multiple models and mathematical reasoning benchmarks, UCPO consistently improves Pass@K while preserving Pass@1, achieving up to a 10\% absolute gain on AIME24 at Pass@64. These results highlight the importance of explicitly shaping the diversity structure of policies in RL for LLMs.

\bibliography{neurips_2026}
\bibliographystyle{plainnat}
%%%%%%%%%%%%%%%%%%%%%%%%%%%%%%%%%%%%%%%%%%%%%%%%%%%%%%%%%%%%
\appendix
\section{Theoretical Details and Proofs}
\input{appendix/Appendix_A_proofs}

\input{appendix/App_Optimal_policy_proofs}
\input{appendix/App_UCPO_proofs}
\input{appendix/choice_of_divergence}
\section{Controlled Empirical Validation of Collapse Dynamics}
\input{appendix/empirical_div_collapse}
\section{Why Global entropy regularization fails}
\input{appendix/why_global_ent_fails}

\section{Experimental Details}
\input{appendix/App_exp_details}
\section{Additional Results}
\input{appendix/App_additional_results}

\end{document}

%% file: sections/introduction.tex
Reinforcement Learning with Verifiable Rewards (RLVR) has emerged as a powerful paradigm for improving large language models (LLMs) on reasoning tasks such as mathematics and code generation, yielding substantial gains in single-attempt accuracy (Pass@1) across challenging reasoning benchmarks \citep{DeepSeekAI2025DeepSeekR1IR, shao2024deepseekmathpushinglimitsmathematical, basereasoningsmater, yue2025does}.

Despite these successes, RLVR often degrades Pass@K even as it improves Pass@1~\citep{he_rewarding, cobbe2021trainingverifierssolvemath,song2025outcome}. This suggests that RL fine-tuning contracts the model’s solution space~\citep{wu2026invisibleleashrlvrescape, basereasoningsmater}, concentrating probability mass on a small subset of correct solutions while suppressing alternative valid reasoning paths present in the base model.

Most prior efforts address this issue \citep{8020rule,he_rewarding,jiang2025rethinking, stabilizingknowledge, entropymech, chen2025pass} 
through entropy bonuses, exploration incentives, or token-level constraints. While these methods can alleviate diversity loss, they share a common limitation: none directly specifies how probability mass should be distributed among correct solutions. Lacking a principled target, they rely on proxies that may satisfy local heuristics (e.g., token-level entropy) while the policy still concentrates mass on a few correct solutions.

In this work, we analyze RLVR from the perspective of \emph{optimal policy structure}. We ask three questions: (i) what policies are optimal under standard RLVR objectives, and what policies are actually attained in practice; (ii) what policy should be considered optimal when multiple correct solutions exist; and (iii) how such a policy can be learned via RL. These questions guide our analysis and the design of a new RL objective.

We first show that the standard RLVR objective is fundamentally underspecified: it is indifferent to how probability mass is distributed among correct solutions. As a result, a fully collapsed policy and a maximally diverse policy will be considered equally optimal. Moreover, the stochastic dynamics of on-policy sampling induce a \emph{self-reinforcing collapse} in which frequently sampled solutions are amplified by gradient updates and suppress alternatives. Thus, diversity collapse arises jointly from objective indifference and training dynamics.

We then study what the optimal policy should be when multiple correct solutions exist. We formalize this through two criteria: robustness optimality (which policy retains the best performance under distribution shift) and entropy-regularized optimality (which policy optimally balances correctness with diversity). Both identify the same solution: the Uniform-Correct Policy, which places zero mass on incorrect outputs while distributing probability uniformly across all correct solutions.

Motivated by this analysis, we propose \emph{Uniform-Correct Policy Optimization} (UCPO), a simple modification of GRPO that admits the Uniform-Correct Policy as its unique optimum. UCPO adds a uniformity penalty over correct solutions, which redistributes gradient signal to amplify underrepresented solutions while tempering dominant ones.

Our contributions are as follows:

\begin{itemize}
    \item We identify a fundamental limitation of RLVR: the objective is indifferent to how probability mass is distributed among correct solutions, allowing stochastic training dynamics to drive irreversible mode collapse. We validate this mechanism both theoretically (Section ~\ref{sec:collapse}) and empirically through controlled experiments (Section ~\ref{sec:finding_1_collpase}).
    \item We theoretically characterize the optimal policy structure for RLVR through two criteria: robustness and entropy-regularized optimality (Section ~\ref{sec:optimal_policy}), proving that the Uniform-Correct Policy is optimal under both criteria.
    \item  We derive Uniform-Correct Policy Optimization (UCPO) to achieve this optimal structure, and prove that the Uniform-Correct Policy is its unique optimum (Theorem ~\ref{thm:stationary}). UCPO preserves the total advantage mass of GRPO but reallocates it toward underrepresented correct responses.
    \item Across three models (1.5B–7B parameters) and five mathematical benchmarks, UCPO consistently improves Pass@K and diversity while maintaining competitive Pass@1, achieving up to +10\% absolute improvement on AIME24 at Pass@64 and up to 45\% higher equation-level diversity within the correct set, with negligible training overhead. 
\end{itemize}

%% file: sections/new_related_works.tex
Prior works have documented diversity collapse in RLVR: fine-tuning improves Pass@1 but can reduce Pass@K by contracting the support of correct outputs and sharpening around high-likelihood base-model modes~\citep{matsutani2025rl, wu2026invisibleleashrlvrescape, he_rewarding, yue2025does, ni2025grpohelpllmstranscend}. A growing body of work aims to address this collapse, which we organize into analysis-driven methods that diagnose a collapse mechanism and derive a targeted fix, and methods that encourage diversity indirectly through token-level stochasticity or trajectory-level exploration.

\subsection{Diagnosing and Addressing Diversity Collapse}
Several concurrent works analyze mechanisms underlying diversity collapse and propose targeted interventions. \citet{entropymech} attribute collapse to entropy decay driven by positive probability-advantage covariance, and propose covariance-based gradient clipping and KL regularization. \citet{gai2025differential} show that selection and reinforcement biases in KL-regularized RLVR amplify base-policy imbalance within the correct set, and propose differential smoothing using sequence log-probabilities\footnote{Their analysis is for KL-regularized RLVR, whereas our focus is on the RLVR objective itself. Throughout our analysis and experiments, we use $\beta=0$, isolating collapse arising from the advantage-maximization term rather than reference regularization.}. \citet{ply2026f} show that easy prompts with high success rates drive the strongest concentration of correctness mass, and propose focal advantage scaling to down-weight updates from these prompts.
LAD \citep{li2026lad} states that expected-advantage maximization distorts relative proportions among correct responses, and replaces it with an 
$f$-divergence objective that preserves the previous policy's proportions. By contrast, our work identifies that RLVR advantage-maximization objectives are indifferent to within-correct set allocation. The proposed UCPO reshapes the within-correct distribution toward uniformity by redistributing gradient mass among correct responses.

 \subsection{Diversity via Token-Level Regularization and Trajectory-Level Exploration}
Entropy-based methods aim to mitigate collapse by maintaining stochasticity during training. Global entropy regularization  acts over the full output space, which can increase probability mass on incorrect responses and conflict with correctness. More targeted token-level methods preserve local stochasticity through entropy bonuses, high-uncertainty token updates, or selective entropy regularization~\citep{cheng2025reasoning,8020rule,stabilizingknowledge,jiang2025rethinking}. However, a model can remain uncertain at the token-level while still collapsing to a narrow set of final solutions \citep{wu2026invisibleleashrlvrescape}. 
Another line of work promotes diversity by broadening trajectory-level exploration through adaptive rollout allocation, targeted exploration, or external data~\citep{liao2025enhancing,fr3e,dong2025agentic}. These methods are complementary to UCPO as they expand the set of reasoning trajectories explored during training, while UCPO operates at the conditional-distribution level, reshaping how gradient mass is allocated among correct responses once they are sampled.

%% file: sections/preliminiaries.tex
We review the Reinforcement Learning with Verifiable Rewards (RLVR) setting and Group Relative Policy Optimization (GRPO).

\subsection{RLVR Objective}

Let $x$ denote a prompt and $y$ a model completion sampled from a policy $\pi_\theta(y \mid x)$. In RLVR, correctness is determined by a verifier $R(x, y) \in \{0, 1\}$, and the set of verifier-accepted (correct) solutions is
$$\mathcal{Y}^+(x) = \{y : R(x, y) = 1\}.$$

The RLVR objective maximizes the probability of producing any correct solution:
\begin{equation}
\begin{aligned}
    J(\pi_\theta) = \mathbb{E}_{x \sim\mathcal{D}} \mathbb{E}_{y \sim \pi_\theta(\cdot|x)}[R(x, y)] = \mathbb{E}_{x\sim\mathcal{D}} \left[ \sum_{y \in \mathcal{Y}^+(x)} \pi_\theta(y|x) \right]
    \label{eq:rlvr_objective}
    \end{aligned}.
\end{equation}
We denote the total probability mass the policy places on correct outputs as
\begin{equation}
Z_\theta(x) = \sum_{y \in \mathcal{Y}^+(x)} \pi_\theta(y \mid x),
\label{eq:correctness_mass}
\end{equation}
so that $J(\pi_\theta) = \mathbb{E}_{x \sim \mathcal{D}}[Z_\theta(x)]$. Maximizing the RLVR objective is therefore equivalent to maximizing $Z_\theta(x)$, or equivalently $\log Z_\theta(x)$.
Taking the gradient of Eq.~\eqref{eq:rlvr_objective} yields the standard policy gradient form:
\begin{equation}
\nabla_\theta J(\pi_\theta)
=
\mathbb{E}_{x\sim\mathcal D,\; y\sim\pi_\theta(\cdot\mid x)}
\left[
R(x,y)\nabla_\theta \log \pi_\theta(y\mid x)
\right].
\label{eq:rlvr_pg}
\end{equation}
The gradient signal depends only on whether 
$y$ is correct, not on how probability mass is distributed among different correct outputs. 
% This property carries through to  algorithms such as GRPO built on this objective.

\subsection{Group Relative Policy Optimization (GRPO)}
\label{sec:grpo_and_collapse}
GRPO~\citep{shao2024deepseekmathpushinglimitsmathematical} estimates RLVR policy gradients using group-relative baselines. For each prompt $x$, it samples $K$ rollouts $\{y_i\}_{i=1}^{K} \sim \pi_{\mathrm{old}}(\cdot \mid x)$ from the rollout policy, obtains rewards $R_i=R(x,y_i)$, and updates the current policy $\pi_\theta$ using importance ratios. Let $\sigma_R = \operatorname{std}(\{R_j\}_{j=1}^K)$ denote the within-group reward standard deviation. The group-relative baseline and advantage are
\begin{equation}
\bar{R} = \frac{1}{K} \sum_{i=1}^{K} R_i,  \quad \quad
A_i = \frac{R_i - \bar{R}}{\sigma_R + \varepsilon}.
\label{eq:grpo_advantage}
\end{equation}

Following PPO \citep{schulman2017proximal}, GRPO uses the clipped surrogate objective:
\begin{equation}
\begin{aligned}
L_{\text{GRPO}}(\theta)
&= \mathbb{E}_{x \sim \mathcal{D}, \{y_i\}_{i=1}^{K} \sim \pi_{\text{old}}(\cdot \mid x)}
\Bigg[
\sum_{i=1}^{K}
\min\Big(
    r_i(\theta) A_i,
    \mathrm{clip}\!\left(r_i(\theta), 1-\varepsilon, 1+\varepsilon\right) A_i
\Big)
\Bigg]  
\label{eq:grpo}
\end{aligned}
\end{equation}
where $r_i(\theta) = \frac{\pi_\theta(y_i|x)}{ \pi_{\mathrm{old}}(y_i|x)}$ is the importance ratio.  All correct responses share advantage $A^+$, and all incorrect responses share $A^-.$ Consequently, GRPO inherits the RLVR objective’s indifference to how probability mass is allocated within the correct set.

%% file: sections/struc_collapse_of_div.tex
\label{sec:the_problem}

We identify a structural limitation of RLVR objectives: they are indifferent to how probability mass is distributed among correct solutions. Combined with optimization dynamics, this indifference leads to systematic diversity collapse.

\subsection{Objective Indifference}
\label{sec:obj_indifference}
The RLVR objective (Eq.~\eqref{eq:rlvr_objective}) specifies what the policy should produce (correct outputs) but not how probability mass should be distributed among multiple correct solutions. The following theorem formalizes this observation.
\begin{theorem}[Stationary Points of RLVR]
\label{thm:rlvr_obj_indifference}
Let $\pi$ be any policy such that $\operatorname{supp}(\pi(\cdot \mid x)) \subseteq \mathcal{Y}^{+}(x)$, i.e., $\pi$ places all probability mass on correct outputs. Then $\pi$ is a stationary point of the RLVR objective (Eq.~\eqref{eq:rlvr_objective}).
\end{theorem}
\begin{proof}
The RLVR objective is $J(\pi) \;=\; \sum_{y \in \mathcal{Y^{+}}(x)} \pi(y \mid x).$
For any policy with $\operatorname{supp}(\pi(\cdot \mid x)) \subseteq \mathcal{Y}^{+}(x)$, we have $J(\pi) = 1$, the maximum possible value. Since $J$ is linear in $\pi$ and the constraint set
\[
\left\{ \pi \;\middle|\; \sum_{y} \pi(y \mid x) = 1,\; \pi(y \mid x) \ge 0 \right\}
\]
is a simplex, any policy achieving $J(\pi) = 1$ is a global maximizer of the RLVR objective.
\end{proof}

\begin{remark}[Implication of Theorem ~\ref{thm:rlvr_obj_indifference}]
Theorem ~\ref{thm:rlvr_obj_indifference} implies that the following policies are all stationary points of RLVR with identical objective value $J(\pi) = 1$: 
\vspace{-0.4em}
\begin{itemize} 
\item \textbf{Collapsed:} $\pi(y_1 \mid x) = 1$ for some $y_1 \in \mathcal{Y}^{+}(x)$, and $\pi(y \mid x) = 0$ otherwise. 
\item \textbf{Uniform-correct:} $\pi(y \mid x) = \frac{1}{|\mathcal{Y}^{+}(x)|}$ for all $y \in \mathcal{Y}^{+}(x)$. 
\vspace{-0.4em}
\item \textbf{Any mixture:} Any distribution supported exclusively on $\mathcal{Y}^{+}(x)$. 
\end{itemize} \end{remark}
Thus, the RLVR objective (and by extension GRPO and other variants) provides no gradient signal to distinguish among these policies once all mass is on correct outputs. The choice among stationary points is determined entirely by optimization dynamics, not by the objective itself.

\subsection{The Collapse Mechanism}
\label{sec:collapse}
Section~\ref{sec:obj_indifference} established that RLVR objectives do not specify how probability mass should be allocated among correct responses. We now show that, under GRPO and related group-relative policy-gradient methods, on-policy sampling resolves this ambiguity by favoring concentration: sampling imbalances are amplified across updates, concentrating mass on a subset of correct responses. We formalize this as a four-step cycle and validate each step in controlled experiments where the full policy distribution is observable (Appendix~\ref{app:toy_exp_collpase}); proofs are provided in Appendix~\ref{app:proofs_for_collpase}.

\subsubsection{Setup}
\label{sec:collapse_setup}
 
Consider training a policy $\pi_\theta$ initialized from $\pi_{\text{ref}}$. For a prompt $x$, let $\mathcal{Y}$ be the finite output space, partitioned into correct responses $\mathcal{Y}^+=\{y_1,\ldots,y_m\}$ and incorrect responses $\mathcal{Y}^-$. Under GRPO, all correct responses share advantage $A^+$ and all incorrect responses share $A^-$.
Throughout this section, $y_i$ denotes a cluster of solutions sharing the same underlying reasoning strategy.

\subsubsection{Sampling Bias}
\label{sec:sampling_bias}
 
The GRPO objective involves an expectation over the policy:
\begin{equation}
\mathcal{L}(\theta) = \mathbb{E}_{y \sim \pi_\theta(\cdot|x)}\left[A(y) \cdot \log \pi_\theta(y|x)\right].
\label{eq:grpo_expectation}
\end{equation}
Since this expectation is intractable, GRPO approximates it via Monte Carlo: drawing $K$ i.i.d.\ rollouts from $\pi_\theta(\cdot|x)$ and computing the empirical gradient. Let $N_y$ denote the number of samples from solution $y$ among the $K$ rollouts.
 
\begin{proposition}[Sampling Bias]
\label{prop:sampling}
Under on-policy sampling, solutions with higher probability are sampled more frequently. For $K$ i.i.d.\ rollouts from $\pi_\theta(\cdot|x)$, the sample count $N_y$ follows:
\begin{equation}
N_y \sim \mathrm{Binomial}(K, \pi_\theta(y|x)), \quad \mathbb{E}[N_y] = K \cdot \pi_\theta(y|x).
\label{eq:sampling}
\end{equation}
\end{proposition}
\begin{proof}
Each rollout independently produces solution $y$ with probability $\pi_\theta(y|x)$. The count of successes in $K$ independent Bernoulli trials is binomial. 
\end{proof}
Higher-probability solutions are therefore overrepresented in sampled batches. Even under a uniform policy, finite-batch multinomial sampling can produce unequal counts across correct responses; Proposition~\ref{prop:symmetry} formalizes this sampling-induced symmetry breaking. Appendix~\ref{app:finding_2_always_skewed} empirically validates this effect in the controlled environment: under uniform initialization, GRPO amplifies whichever correct response is overrepresented, leading to collapse to an arbitrary winner. The key question is therefore whether the subsequent gradient step corrects this imbalance or amplifies it.

\subsubsection{Gradient Amplification }
\label{sec:gradient_amplification}
Since sampling produces unequal counts $\{N_y\}$, we now examine how this imbalance enters the gradient update.
The GRPO gradient over the batch is:
\begin{equation}
\nabla_\theta \mathcal{L} = \sum_{i=1}^{K} A_i \cdot \nabla_\theta \log \pi_\theta(y_i|x).
\label{eq:grpo_gradient}
\end{equation}
Grouping by solution, the gradient contribution of each correct $y \in \mathcal{Y}^+$ is: $N_y \cdot A^+ \cdot \nabla_\theta \log \pi_\theta(y|x).$

\begin{theorem}[Gradient Amplification]
\label{thm:gradient_amplification}
Under GRPO with on-policy sampling, the expected gradient contribution from a correct solution $y \in \mathcal{Y}^+$ is:
\begin{equation}
\mathbb{E}\left[\sum_{i: y_i = y} A^+ \nabla_\theta \log \pi_\theta(y|x)\right] = A^+ \cdot K \cdot \pi_\theta(y|x) \cdot \nabla_\theta \log \pi_\theta(y|x).
\label{eq:amplification}
\end{equation}
\end{theorem}
Thus, expected gradient contribution scales with current probability: high-probability solutions receive larger updates, while solutions with $\pi_\theta(y\mid x)\ll 1/K$ are likely unsampled and lose mass through softmax renormalization.

\subsubsection{Divergence of Probability Ratios}
\label{sec:divergence}
The combination of sampling bias and gradient amplification causes probability ratios to diverge over training.

\begin{theorem}[Compounding Divergence]
\label{thm:divergence}
Let $\pi_t$ denote the policy after $t$ updates. Under GRPO, for two correct solutions $y_1, y_2 \in \mathcal{Y}^+(x)$ with $\pi_0(y_1|x) > \pi_0(y_2|x)$, the expected change in their log-probability ratio at step $t$ is:
\begin{equation}
\mathbb{E}\left[\Delta \log \frac{\pi_t(y_1|x)}{\pi_t(y_2|x)}\right] = \eta \cdot A^+ \cdot K \cdot \bigl(\pi_t(y_1|x) - \pi_t(y_2|x)\bigr) > 0,
\label{eq:ratio_divergence}
\end{equation}
where $\eta$ is the effective learning rate. The log-ratio is a submartingale with positive drift that increases as the probability gap widens.
\end{theorem}
This feedback causes small initial asymmetries to compound: the dominant solution receives larger updates, increasing its probability and expected sample count in subsequent batches, widening the gap and reinforcing the collapse.

\subsubsection{Irreversible Pruning}

The collapse becomes irreversible once minority solutions fall below the sampling threshold.

\begin{corollary}[Sampling Threshold]\label{cor:threshold}
When $\pi_t(y|x) \ll 1/K$, the probability that solution $y$ receives at least one sample in a batch of size $K$ is:
\[
P(N_y \geq 1) = 1 - (1 - \pi_t(y|x))^K \approx K \cdot \pi_t(y|x) \ll 1.
\]
Below this threshold, solution $y$ receives effectively zero gradient signal and cannot recover under on-policy sampling.
\end{corollary}

The four steps trace a complete trajectory: (i) initial asymmetries arise from sampling bias, (ii) GRPO amplifies these asymmetries into the gradient,  (iii) the updated policy increases the probability gap, which worsens sampling imbalance in subsequent iterations, and (iv) once low-probability solutions fall below the sampling threshold, recovery becomes unlikely. This mechanism explains the empirical Pass@1 versus Pass@$K$ trade-off( Figure~\ref{fig:div_collpase_base_vs_grpo}): GRPO improves single-sample accuracy while reducing multi-sample coverage.

\subsection{Controlled Experiment}
\label{sec:finding_1_collpase}
To validate the proposed collapse mechanism, we construct a controlled RLVR environment with a softmax policy over a finite discrete output space, of which three outputs are correct and the remainder are incorrect. This makes the full policy distribution exactly observable throughout training (details in Appendix~\ref{app:toy_exp_collpase}). Figure~\ref{fig:collapse_figure} confirms all four stages of the collapse cycle; for details see Appendix~\ref{app:finding_1_collpase}. 
\begin{figure*}[!ht]
        \centering
        \includegraphics[width=\linewidth]{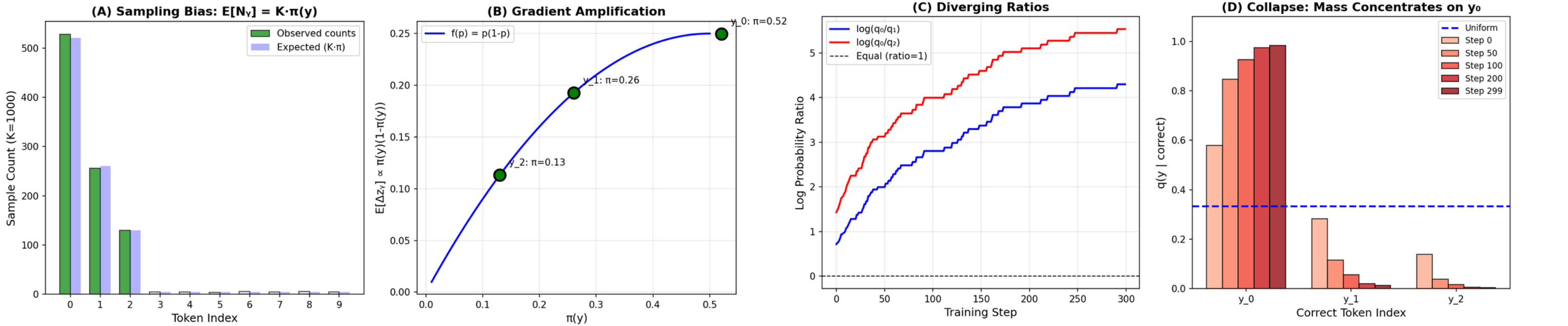}
    \caption{Controlled verification of the GRPO collapse mechanism in the toy RLVR environment.
(A) On-policy sampling over-represents high-probability correct tokens, matching $E[N_y]=K\cdot \pi(y).$
(B) Because gradient contribution scales with sample count, dominant correct tokens receive larger expected updates.
(C) The log-ratio between dominant and minority correct tokens increases over training, confirming compounding divergence.
(D) Probability mass eventually concentrates on the initially preferred correct token, illustrating mode collapse within the correct set.}
    \label{fig:collapse_figure}
\end{figure*}
\vspace{-10pt}

%% file: sections/optimal_policy.tex
\label{sec:optimal_policy}
We established that GRPO induces systematic collapse toward dominant correct solutions, we now ask: \textit{what should the optimal policy look like when multiple correct solutions exist?} 
We formalize this question through the conditional distribution over correct outputs and prove that a specific structure, uniform probability across all correct solutions, is optimal under two independent criteria.

\subsection{Conditional Distribution on Correct Outputs}

\begin{tcolorbox}[
    colback=green!4,
    colframe=green!50!black,
    title=\textbf{Conditional Correct Distribution},
    fonttitle=\bfseries,
    boxrule=0.6pt,
    arc=2mm,
    left=6pt,
    right=6pt,
    top=2pt,
    bottom=0pt
]
For a policy $\pi$, define the conditional distribution over correct outputs:
\begin{equation}
    q(y|x) = \frac{\pi(y|x)}{Z(x)}, \quad \text{where} \quad Z(x) = \sum_{y \in \mathcal{Y^{+}}(x)} \pi(y|x).
    \label{eq:conditional_distribution}
\end{equation}
\end{tcolorbox}

Here $Z(x)$ is the total probability that the policy produces a correct answer. While RLVR optimizes $Z(x)$, it leaves the structure of $q(y \mid x)$ unconstrained. Diversity collapse corresponds to $q$ concentrating on a small subset of $\mathcal{Y^{+}}$.

\subsection{The Uniform-Correct Policy}
\label{sec:uniform_correct_policy}
\begin{tcolorbox}[
    colback=green!4,
    colframe=green!50!black,
    title=\textbf{Uniform-Correct Policy (UCP)},
    fonttitle=\bfseries,
    boxrule=0.6pt,
    arc=2mm,
    left=6pt,
    right=6pt,
    top=4pt,
    bottom=2pt
]
% \begin{definition}[Uniform-Correct Policy]
A policy $\pi^*$ is \emph{uniform-correct} if:
\begin{enumerate}
    \item $\pi^*(y|x) = 0$ for all $y \notin \mathcal{Y^{+}}(x)$ 
    %(no mass on incorrect outputs)
    \item $q^*(y|x) = 1/|\mathcal{Y^{+}}|$ for all $y \in \mathcal{Y^{+}}(x)$ 
    %(uniform among correct)
\end{enumerate}
% \end{definition}
\end{tcolorbox}
This policy places all probability mass on correct outputs and distributes it uniformly among them.
We show that this structure is optimal under two complementary criteria: Robustness under distribution shift, which guards against test-time conditions that differ from training, and Entropy-regularized Optimality, a standard theoretical principle for resolving reward indifference (proofs in Appendix~\ref{app:optimal_policy_proofs}). 
\begin{theorem}[Robustness under distribution shift]
\label{thm:robustness}
Suppose a distribution shift disables a subset $S \subset \mathcal{Y^{+}}$ of correct outputs, where $|S| \leq s$. Consider the robust objective:
\begin{equation}
    \max_\pi \min_{S: |S| \leq s} \sum_{y \in \mathcal{Y^{+}} \setminus S} \pi(y|x).
\end{equation}
Then the Uniform-Correct Policy uniquely maximizes the worst-case retained probability mass.
\end{theorem}

\begin{theorem}
\label{thm:entropy_regularized}
Consider the entropy-regularized objective
\begin{equation}
\max_{\pi(\cdot \mid x)} \;
\mathbb{E}_{y \sim \pi}[R(x,y)]
+ \tau \, \mathrm{H}(\pi(\cdot \mid x)).
\label{eq:entropy_regularized_objective}
\end{equation}
There exists a threshold $\tau_0 > 0$ such that for all $0 < \tau < \tau_0$, the UCP uniquely maximizes Eq.~\eqref{eq:entropy_regularized_objective}.
\end{theorem}
Theorem~\ref{thm:robustness} establishes that the Uniform-Correct Policy is the minimax-optimal solution against the loss of any bounded subset of correct outputs, and Theorem~\ref{thm:entropy_regularized} shows that it is the unique solution to the entropy-regularized objective for sufficiently small $\tau$.

%% file: sections/ucpo.tex
\label{sec:UCPO}
Section~\ref{sec:optimal_policy} established that the uniform-correct policy is uniquely optimal. We now derive an objective that targets this policy during training.

\subsection{UCPO Motivation}
\label{sec:ucpo}

A direct way to operationalize Theorem~\ref{thm:entropy_regularized} is global entropy regularization. However, the optimization dynamics of this approach hinder convergence to the UCP. Because global entropy acts over the full output space, it can increase mass on incorrect outputs rather than controlling the conditional distribution over correct responses. Empirically, it yields small and less consistent Pass@K gains and is sensitive to the entropy coefficient; see Appendix~\ref{app:global-entropy-fails} for details.

This motivates an objective that directly controls how the policy allocates mass among correct responses.
We therefore propose \emph{Uniform-Correct Policy Optimization} (UCPO), which augments correctness-mass maximization with a conditional uniformity penalty that encourages the policy to match the uniform distribution over correct responses. UCPO therefore increases total correctness mass while promoting a balanced allocation of probability mass within the correct set.

\subsection{Objective}

For a prompt $x\sim\mathcal{D}$, let $\mathcal{Y}^+(x)$ be the set of correct responses.
During training, we sample $N$ rollouts from $\pi_\theta (\cdot|x)$. The correct rollouts $y \in \mathcal{Y}^+(x)$ are effectively samples from the conditional policy $q_\theta(\cdot|x)$: 
\begin{equation}
q_\theta(y \mid x) = \frac{\pi_\theta(y \mid x)}{Z_\theta(x)},
\qquad Z_\theta(x) = \sum_{y \in \mathcal{Y}^+(x)} \pi_\theta(y \mid x)
\notag
\end{equation}
where $Z_\theta(x)$ is the total probability mass that the model places on all the correct responses. 
We also define the uniform distribution over the correct set:
\[u(y \mid x) = \frac{1}{|\mathcal{Y}^+(x)|}.
\] 

We propose the \emph{Uniform-Correct Policy Optimization} (UCPO) objective:
\begin{equation}
\label{eq:ucpo-objective}
L_{\mathrm{UCPO}}(\theta) = \mathbb{E}_{x \sim \mathcal{D}} \left[
\underbrace{\log Z_\theta(x)}_{\text{correctness mass}}
- \tau\,\underbrace{{\color{UCPO}\mathrm{KL}(u(\cdot\mid x)\|q_\theta(\cdot\mid x)}}_{\text{uniformity penalty}}
\right]
\end{equation}
Relative to GRPO, UCPO introduces a single additional term: the conditional uniformity penalty $\mathrm{KL}(u \| q_\theta)$. At $\tau = 0$, UCPO reduces to GRPO. For $\tau > 0$, the added term breaks GRPO’s indifference to within-correct mass allocation and biases the policy toward the uniform-correct target. In particular, the Uniform-Correct Policy is the unique optimum of the UCPO objective (Theorem~\ref{thm:stationary}). This modification incurs negligible additional computational overhead, since the penalty can be estimated from the same verifier-accepted rollouts already used by GRPO. As $\tau$ increases, the penalty increasingly discourages concentration, pushing $q_\theta$ toward the uniform target established in Section~\ref{sec:optimal_policy}.

\begin{theorem}[Uniform-correct policy is the unique optimum of UCPO]
\label{thm:stationary}
Fix a prompt $x$ and consider the UCPO objective:
\[
L_{\mathrm{UCPO}}(x)
=
\log Z_\theta(x)
-
\tau\,\mathrm{KL}\!\bigl(u(\cdot\mid x)\,\|\,q_\theta(\cdot\mid x)\bigr).
\]
For any $\tau>0$, the unique maximizer of $L_{\mathrm{UCPO}}(x)$ over policies $\pi_\theta(\cdot\mid x)$ is the Uniform-Correct Policy.
\end{theorem}

\subsection{Gradient Estimation}

UCPO counteracts collapse through its gradient structure. Although GRPO is indifferent to within-correct allocation at the objective level, its sampled update effectively weights correct responses according to the current conditional distribution $q_\theta$. UCPO instead interpolates between $q_\theta$ and the uniform target, shifting gradient mass away from dominant correct solutions and toward underrepresented ones.

\begin{proposition}[Gradient Decomposition]
\label{prop:grad-decomp}
For any prompt $x$ with $\mathcal{Y}^+(x) \neq \emptyset$,
\begin{equation}
\label{eq:grad-decomp}
\nabla_\theta L_{\mathrm{UCPO}}(x) = \sum_{y \in \mathcal{Y}^+(x)} \Big[ (1-\tau)\, q_\theta(y \mid x) + \tau\, u(y \mid x) \Big] \nabla_\theta \log \pi_\theta(y \mid x).
\end{equation}
\end{proposition}

\noindent The proof is provided in Appendix~\ref{app:grad-proof}. The first term is estimated directly from the $n$ correct rollouts. The second term requires samples from the uniform distribution $u$, which we estimate via importance sampling. Since $u(y)/q_\theta(y) \propto 1/q_\theta(y)$, we define inverse weights $v_i = 1/\hat{q}_i$, where
\[
\hat{q}_i = \frac{\pi_\theta(y_i \mid x)}{\sum_{j=1}^{n} \pi_\theta(y_j \mid x)},
\]
with the derivation given in Appendix~\ref{app:snis}. The resulting estimator is
\[
\nabla_\theta L_{\mathrm{UCPO}}
\approx \sum_{i=1}^{n} w_i \nabla_\theta \log \pi_\theta(y_i \mid x),
\qquad
w_i = \frac{1-\tau}{n} + \tau \cdot \frac{v_i}{\sum_{j=1}^n v_j}.
\]
By construction, $\sum_i w_i = 1$. Consequently, low-probability correct solutions receive larger weights, directly counteracting the sampling bias identified in Section~\ref{sec:collapse}.

Equivalently, UCPO redistributes GRPO's total advantage budget $n\cdot A^+$ according to $w_i$, giving larger updates to underrepresented correct solutions: 
\begin{equation}
\label{eq:ucpo-advantage}
A_i^{\mathrm{UCPO}} = n \cdot A^+ \cdot w_i.
\end{equation}

\begin{proposition}[Mass Preservation]
\label{prop:mass}
$\sum_{i=1}^n A_i^{\mathrm{UCPO}} = \sum_{i=1}^n A_i^{\mathrm{GRPO}}$.
\end{proposition}

\noindent Thus, UCPO preserves the total advantage mass assigned to correct rollouts and reallocates it towards underrepresented correct responses. The proof is provided in Appendix~\ref{app:proofs}.

\input{tables/res_ds_combo}

%% file: tables/res_ds_combo.tex
\begin{table*}[t]
\centering
\scriptsize
\setlength{\tabcolsep}{4pt}
\renewcommand{\arraystretch}{1.15}
\caption{Pass@1 and Pass@64 (\%) across models and reasoning benchmarks. UCPO improves average Pass@64 across all three model families while remaining competitive at Pass@1. Best results are shown in bold.}
\label{tab:deepseek_passk_combined}

\textbf{DeepSeek-Distill-Qwen-1.5B}

\vspace{2pt}
\begin{tabular}{lcc@{\hspace{4pt}}cc@{\hspace{4pt}}cc@{\hspace{4pt}}cc@{\hspace{4pt}}cc@{\hspace{4pt}}cc}
\toprule
& \multicolumn{2}{c}{\textbf{AIME24}}
& \multicolumn{2}{c}{\textbf{AIME25}}
& \multicolumn{2}{c}{\textbf{AMC}}
& \multicolumn{2}{c}{\textbf{MATH}}
& \multicolumn{2}{c}{\textbf{OlympiadBench}}
& \multicolumn{2}{c}{\textbf{Avg}} \\
\cmidrule(lr){2-3}\cmidrule(lr){4-5}
\cmidrule(lr){6-7}\cmidrule(lr){8-9}
\cmidrule(lr){10-11}\cmidrule(lr){12-13}
\textbf{Method} 
& \textbf{@1} & \textbf{@64} 
& \textbf{@1} & \textbf{@64} 
& \textbf{@1} & \textbf{@64} 
& \textbf{@1} & \textbf{@64} 
& \textbf{@1} & \textbf{@64} 
& \textbf{@1} & \textbf{@64} \\
\midrule
GRPO    & 19.06 & 63.33 & \textbf{20.21} & 46.67 & 55.27 & 87.95 & 71.57 & \textbf{97.60} & 42.63 & 78.52 & \textbf{41.75} & 74.81 \\
KL-Cov  & \textbf{19.64} & 60.00 & 18.13 & 43.33 & 54.78 & 91.57 & 70.47 & 96.80 & 42.69 & 78.22 & 41.14 & 73.98 \\
Ent-Reg & 19.43 & 63.33 & 17.24 & 46.67 & \textbf{56.12} & 89.16 & \textbf{72.23} & \textbf{97.60} & \textbf{43.03} & 78.07 & 41.61 & 74.97 \\
Ent-Adv & 19.58 & 63.33 & 15.94 & 46.67 & 54.07 & 91.57 & 69.37 & 97.00 & 41.59 & 79.11 & 40.11 & 75.54 \\
Pass@K  & 15.47 & \textbf{66.67} & 14.27 & \textbf{53.33} & 47.63 & \textbf{93.98} & 62.74 & 97.40 & 35.24 & 77.93 & 35.07 & 77.86 \\
FGRPO   & 17.92 & \textbf{66.67} & 16.35 & 50.00 & 52.79 & \textbf{93.98} & 67.76 & 97.20 & 40.57 & 78.96 & 39.08 & 77.36 \\
\rowcolor{UCPO!20}
UCPO    & 19.01 & \textbf{66.67} & 18.49 & \textbf{53.33} & 52.97 & \textbf{93.98} & 67.93 & 97.20 & 40.15 & \textbf{80.15} & 39.71 & \textbf{78.27} \\
\bottomrule
\end{tabular}

\vspace{6pt}
\textbf{DeepSeek-Distill-Qwen-7B}

\vspace{2pt}
\begin{tabular}{lcc@{\hspace{4pt}}cc@{\hspace{4pt}}cc@{\hspace{4pt}}cc@{\hspace{4pt}}cc@{\hspace{4pt}}cc}
\toprule
& \multicolumn{2}{c}{\textbf{AIME24}}
& \multicolumn{2}{c}{\textbf{AIME25}}
& \multicolumn{2}{c}{\textbf{AMC}}
& \multicolumn{2}{c}{\textbf{MATH}}
& \multicolumn{2}{c}{\textbf{OlympiadBench}}
& \multicolumn{2}{c}{\textbf{Avg}} \\
\cmidrule(lr){2-3}\cmidrule(lr){4-5}
\cmidrule(lr){6-7}\cmidrule(lr){8-9}
\cmidrule(lr){10-11}\cmidrule(lr){12-13}
\textbf{Method} 
& \textbf{@1} & \textbf{@64} 
& \textbf{@1} & \textbf{@64} 
& \textbf{@1} & \textbf{@64} 
& \textbf{@1} & \textbf{@64} 
& \textbf{@1} & \textbf{@64} 
& \textbf{@1} & \textbf{@64} \\
\midrule
GRPO    & 27.71 & 70.00 & 22.08 & 53.33 & 61.86 & 93.98 & 77.99 & 98.20 & 47.28 & 82.96 & 47.38 & 79.69 \\
KL-Cov  & \textbf{31.15} & 76.67 & \textbf{23.07} & \textbf{56.67} & \textbf{63.42} & 95.18 & 78.21 & 98.20 & 47.99 & 81.93 & \textbf{48.77} & 81.73 \\
Ent-Reg & 29.06 & 76.67 & 21.30 & \textbf{56.67} & 63.08 & 93.98 & \textbf{78.31} & \textbf{99.00} & \textbf{48.26} & 82.37 & 48.00 & 81.74 \\
Ent-Adv & 29.06 & \textbf{80.00} & 21.15 & \textbf{56.67} & 59.71 & 96.39 & 76.08 & \textbf{99.00} & 45.92 & \textbf{83.56} & 46.38 & 83.12 \\
Pass@K  & 23.39 & 70.00 & 19.11 & \textbf{56.67} & 56.44 & 93.98 & 72.59 & 98.20 & 41.77 & 80.74 & 42.66 & 79.92 \\
FGRPO   & 27.55 & 76.67 & 21.25 & 53.33 & 61.77 & 95.18 & 76.25 & 98.80 & 46.47 & 81.93 & 46.66 & 81.18 \\
\rowcolor{UCPO!20}
UCPO    & 30.16 & \textbf{80.00} & 21.51 & \textbf{56.67} & 62.58 & \textbf{97.59} & 76.45 & \textbf{99.00} & 47.39 & 83.26 & 47.62 & \textbf{83.30} \\
\bottomrule
\end{tabular}

\vspace{4pt}
\textbf{Qwen2.5-7B}
\vspace{1pt}

\begin{tabular}{lcc@{\hspace{4pt}}cc@{\hspace{4pt}}cc@{\hspace{4pt}}cc@{\hspace{4pt}}cc@{\hspace{4pt}}cc}
\toprule
& \multicolumn{2}{c}{\textbf{AIME24}}
& \multicolumn{2}{c}{\textbf{AIME25}}
& \multicolumn{2}{c}{\textbf{AMC}}
& \multicolumn{2}{c}{\textbf{MATH}}
& \multicolumn{2}{c}{\textbf{OlympiadBench}}
& \multicolumn{2}{c}{\textbf{Avg}} \\
\cmidrule(lr){2-3}\cmidrule(lr){4-5}
\cmidrule(lr){6-7}\cmidrule(lr){8-9}
\cmidrule(lr){10-11}\cmidrule(lr){12-13}

\textbf{Method} 
& \textbf{@1} & \textbf{@64} 
& \textbf{@1} & \textbf{@64} 
& \textbf{@1} & \textbf{@64} 
& \textbf{@1} & \textbf{@64} 
& \textbf{@1} & \textbf{@64} 
& \textbf{@1} & \textbf{@64} \\
\midrule
GRPO      
& 14.27 & 43.33 & \textbf{5.26} & 40.00 & 44.26 & 90.36 & \textbf{71.61} & 95.40 & 34.85 & 72.74 
& \textbf{34.05} & 68.37 \\

KL-Cov    
& 12.50 & 46.67 & 4.27 & 46.67 & 41.87 & 87.95 & 67.58 & 95.40 & 31.84 & 74.81 
& 31.61 & 70.30 \\

Ent-Reg   
& 11.98 & 50.00 & 4.11 & 40.00 & 42.09 & 86.75 & 66.60 & 94.80 & 32.02 & 73.48 
& 31.36 & 69.01 \\

Ent-Adv   
& 12.71 & \textbf{60.00} & 4.69 & 43.33 & 41.74 & 90.36 & 68.86 & 95.00 & 32.41 & 73.33 
& 32.08 & 72.40 \\

Pass@K    
& 10.73 & 53.33 & 4.01 & 43.33 & 38.52 & 89.16 & 61.85 & 95.60 & 28.82 & 74.22 
& 28.79 & 71.13 \\

FGRPO     
& 12.97 & 56.67 & 5.05 & 40.00 & 41.74 & \textbf{91.57} & 67.02 & 95.40 & 32.21 & 74.37 
& 31.80 & 71.60 \\
\rowcolor{UCPO!20}
UCPO      
& \textbf{15.05} & 56.67 & 5.21 & \textbf{46.67} & \textbf{43.66} & \textbf{91.57} & 70.74 & \textbf{95.80} & \textbf{35.03} & \textbf{75.26} 
& 33.94 & \textbf{73.19} \\
\bottomrule
\end{tabular}

\end{table*}
% \vspace{-35pt}

%% file: sections/experiments.tex
In this section, we evaluate UCPO across multiple model families, parameter scales, and mathematical reasoning benchmarks. 
\vspace{-10pt}
\subsection{Setup}

We train on 10K competition-style math problems from the DeepScaleR corpus \citep{luo2025deepscaler, liao2025enhancing}, with 500 reserved for validation. We evaluate on three models: DeepSeek-R1-Distill-Qwen-1.5B, DeepSeek-R1-Distill-Qwen-7B \citep{DeepSeekAI2025DeepSeekR1IR}, and Qwen2.5-7B \citep{qwen2.5}. All models are fine-tuned from public checkpoints using the same RLVR pipeline.

We compare UCPO against standard GRPO and five diversity-preserving baselines with varying strategies for mitigating collapse in RLVR:  adding entropy to the  objective (Ent-Reg), adding an entropy-based advantage bonus (Ent-Adv)~\citep{cheng2025reasoning}, applying covariance-based KL control (KL-Cov)~\citep{entropymech}, optimizing a Pass@K-oriented reward with analytical advantages (Pass@K Training)~\citep{chen2025pass}, and using focal-style prompt-level advantage scaling (FGRPO). Dataset composition and implementation details are in Appendix~\ref{app:implementation_det}.

We evaluate on AIME 2024/2025 \citep{li2024numinamath}, AMC 2023 \citep{li2024numinamath}, MATH 500 \citep{gao2024omni}, and OlympiadBench \citep{he2024olympiadbench}, reporting Pass@1 and Pass@K (K=64). 
Pass@K captures coverage over distinct correct solutions, making it suitable for evaluating diversity. 

\subsection{Main Results}

Table~\ref{tab:deepseek_passk_combined} reports Pass@1 and Pass@64 across 5 benchmarks and 3 models  (1.5B to 7B parameters). Across all settings, GRPO improves Pass@1 but saturates at larger $K$, indicating concentration of probability mass on a small subset of correct solutions. UCPO maintains competitive Pass@1 while achieving consistent improvements at larger $K$.

On DeepSeek-Distill-Qwen-1.5B, UCPO improves average Pass@64 to 78.27\%, outperforming GRPO by +3.5 points while remaining comparable at Pass@1. Gains are particularly strong on AIME25 (+6.7) and AMC (+6.0) at $K{=}64$.
On DeepSeek-Distill-Qwen-7B, UCPO improves average Pass@64 by +3.6 points, with +10.0 on AIME24. On Qwen2.5-7B, UCPO achieves the highest average Pass@64 (73.19\%, +4.8 over GRPO) while also maintaining the strong Pass@1. These gains indicate that UCPO improves coverage over correct solutions rather than trading off accuracy for diversity.

\begin{table}[!ht]
\centering
\scriptsize
\setlength{\tabcolsep}{4pt}
\renewcommand{\arraystretch}{0.95}
\caption{Equation-level diversity among correct solutions. UCPO achieves the highest diversity across benchmarks.}
\label{tab:eq_diversity}
\begin{tabular}{lcccccc>{\columncolor{UCPO!20}}c}
\toprule
\textbf{Benchmark} & \textbf{GRPO} & \textbf{KL-Cov} & \textbf{Ent-Reg} & \textbf{Ent-Adv} & \textbf{Pass@K} & \textbf{FGRPO} & \textbf{UCPO} \\
\midrule
AIME     & 0.150 & 0.179 & 0.157 & 0.186 & 0.167 & 0.168 & \cellcolor{UCPO!20}\textbf{0.218} \\
AMC      & 0.197 & 0.176 & 0.155 & 0.188 & 0.190 & 0.192 & \cellcolor{UCPO!20}\textbf{0.217} \\
MATH     & 0.173 & 0.156 & 0.140 & 0.159 & 0.182 & 0.171 & \cellcolor{UCPO!20}\textbf{0.196} \\
Olympiad & 0.208 & 0.185 & 0.163 & 0.193 & 0.210 & 0.203 & \cellcolor{UCPO!20}\textbf{0.228} \\
\midrule
Avg.     & 0.182 & 0.174 & 0.154 & 0.181 & 0.187 & 0.184 & \cellcolor{UCPO!20}\textbf{0.215} \\
\bottomrule
\end{tabular}
\end{table}
\vspace{-10pt}

\subsection{Diversity}
\label{sec:diversity_results}
To directly measure diversity within the correct set, we evaluate equation-level diversity among verifier-accepted solutions (Table~\ref{tab:eq_diversity}). This metric computes the fraction of mathematical expressions in a correct rollout that are unique relative to other correct rollouts for the same prompt~\citep{hu2025diversity}. Higher values indicate more distinct reasoning trajectories.

UCPO consistently achieves the highest diversity across all benchmarks, with substantial improvements over GRPO (e.g., +45\% on AIME, +13\% on MATH). In contrast, entropy-based methods do not reliably improve diversity within the correct set and in some cases reduce it. This confirms that UCPO's Pass@K gains arise from recovering distinct correct reasoning trajectories, not from spreading mass indiscriminately. This is consistent with the controlled experiments in Appendix~\ref{app:toy_exp_collpase}: GRPO collapses the conditional distribution over correct solutions, while UCPO maintains a more uniform allocation.

\subsection{Pass@K Performance}
Figures~\ref{fig:passk_all5} show Pass@K curves for all three models and reasoning benchmarks. UCPO consistently improves performance at larger $K$ while remaining competitive at Pass@1, particularly on challenging benchmarks (AIME 2024/2025). This shows UCPO increases coverage over correct reasoning trajectories rather than concentrating on a single dominant mode. By redistributing gradient signal within the verifier-accepted set, UCPO prevents a small number of high-probability solutions from dominating learning, yielding sustained Pass@K gains without shifting mass toward incorrect outputs.

\subsection{Cost and Hyperparameter Sensitivity}
UCPO introduces a single hyperparameter $\tau$ beyond standard GRPO. Its training cost is effectively identical to GRPO (Appendix~\ref{app:cost}). UCPO is also robust across a moderate range of $\tau$ values, whereas global entropy regularization is more sensitive to its coefficient (Appendix~\ref{app:tau_sensitivity}).

%% file: appendix/Appendix_A_proofs.tex
\subsection{The Diversity Collapse Mechanism}
\label{app:proofs_for_collpase}
\subsubsection*{Sampling Bias: Connection to Pólya Urn and Preferential Attachment}
\label{app:preferential_attachment}
The sampling bias in the collapse mechanism is mathematically related to classical models of preferential attachment:
Consider \textbf{Pólya Urn Analogy} of an urn containing balls of $M$ colors (representing $M$ correct solutions), with initial counts proportional to $\{\pi_0(y_i)\}$. At each step, draw $K$ balls with replacement, observe their colors, then add balls proportionally to the observed counts. This is the Pólya urn process, known to converge to a random limit where one color dominates with positive probability.

GRPO's on-policy sampling implements an analogous dynamic: the ``urn composition'' (policy distribution) is updated proportionally to observed samples (gradient updates), creating identical rich-get-richer dynamics.

\subsubsection*{Unequal sample counts are inevitable }
\label{app:unequal_samples_always}
Whether or not the policy initially assigns equal probability to correct solutions, the sample counts $\{N_y\}_{y \in \mathcal{Y}^+}$ will be unequal:
\begin{itemize}
    \item \textbf{Deterministic case:} If $\pi(y_1|x) > \pi(y_2|x)$, then $\mathbb{E}[N_{y_1}] > \mathbb{E}[N_{y_2}]$. Higher-probability solutions are systematically over-represented.
    \item \textbf{Stochastic case:} Even if $\pi(y|x) = 1/m$ for all $y \in \mathcal{Y}^+$, the realized counts will differ due to sampling variance (see Proposition~\ref{prop:symmetry} and empirical illustration in Appendix ~\ref{app:toy_exp_collpase}).
\end{itemize}
\begin{proposition}[Stochastic Symmetry Breaking]
\label{prop:symmetry}
Suppose $\pi_0(y|x) = 1/m$ for all $y \in \mathcal{Y}^+(x)$. For any finite batch size $K$, with positive probability, the sample counts $\{N_y\}$ are not all equal.
\end{proposition}
 
\begin{proof}
Under a uniform policy over the $m$ correct responses, the count vector satisfies
\[
(N_{y_1},\ldots,N_{y_m})
\sim
\mathrm{Multinomial}\!\left(K,\left(\frac{1}{m},\ldots,\frac{1}{m}\right)\right).
\]
If $K$ is not divisible by $m$, exact equality of all counts is impossible. If $K$ is divisible by $m$, exact equality requires $N_{y_j}=K/m$ for every $j$, which occurs with probability
\[
\mathbb{P}(N_{y_1}=\cdots=N_{y_m}=K/m)
=
\frac{K!}{((K/m)!)^m}\cdot \frac{1}{m^K}
<1.
\]
Thus, unequal counts occur with positive probability in any finite batch. On this event, at least one correct response is overrepresented, i.e., there exists $y^*$ such that $N_{y^*} > K/m$. Across repeated independent batches, the probability that no imbalance ever occurs is zero, so symmetry-breaking occurs almost surely over training.
\end{proof}

\subsection{Proof of Theorem ~\ref{thm:gradient_amplification}, Theorem ~\ref{thm:divergence} and Corollary~\ref{cor:threshold}}

\begin{theorem}[Gradient Amplification]
\label{app:proof_amplification}
Under GRPO with on-policy sampling, the expected gradient contribution from a correct solution $y \in \mathcal{Y}^+$ is:
\begin{equation}
\mathbb{E}\left[\sum_{i: y_i = y} A^+ \nabla_\theta \log \pi_\theta(y|x)\right] = A^+ \cdot K \cdot \pi_\theta(y|x) \cdot \nabla_\theta \log \pi_\theta(y|x).
\label{eq:amplification}
\end{equation}
The expected contribution is proportional to current probability $\pi(y|x)$. Solutions with higher probability receive larger expected updates; solutions with $\pi(y|x) \ll 1/K$ likely remain unsampled and receive no gradient signal, causing their probability to decrease through softmax renormalization.
\end{theorem}
\begin{proof}

The GRPO gradient aggregated over the batch is:
\[
\nabla_\theta \mathcal{L} = \sum_{i=1}^{K} A_i \cdot \nabla_\theta \log \pi_\theta(y_i|x).
\]
For correct solutions ($y_i \in \mathcal{Y}^+$), the advantage is $A_i = A^+$. Grouping terms by solution identity:
\[
\nabla_\theta \mathcal{L}\big|_{\mathcal{Y}^+} = \sum_{y \in \mathcal{Y}^+} N_y \cdot A^+ \cdot \nabla_\theta \log \pi_\theta(y|x),
\]
where $N_y = \sum_{i=1}^K \mathbf{1}[y_i = y]$ counts occurrences of $y$ in the batch.

Taking expectations and applying Proposition~\ref{prop:sampling}:
\begin{align}
\mathbb{E}\left[N_y \cdot A^+ \cdot \nabla_\theta \log \pi_\theta(y|x)\right] 
&= A^+ \cdot \mathbb{E}[N_y] \cdot \nabla_\theta \log \pi_\theta(y|x) \\
&= A^+ \cdot K \cdot \pi_\theta(y|x) \cdot \nabla_\theta \log \pi_\theta(y|x).
\end{align}

The expected gradient contribution is proportional to $\pi_\theta(y|x)$:
\begin{itemize}
    \item High-probability solutions: $\pi_\theta(y|x)$ large $\Rightarrow$ large expected contribution
    \item Low-probability solutions: $\pi_\theta(y|x)$ small $\Rightarrow$ small expected contribution  
    \item Very low probability ($\pi_\theta(y|x) \ll 1/K$): $\mathbb{E}[N_y] \ll 1$, so $y$ is likely unsampled and contributes nothing
\end{itemize}
For an unsampled solution $y_{\text{unsamp}}$:
\begin{equation}
\Delta z_{y_{\text{unsamp}}} = 0 \quad \text{(no gradient signal)}
\end{equation}

After the update, the softmax normalization means:
\begin{equation}
\pi'_\theta(y_{\text{unsamp}}|x) = \frac{\exp(z_{y_{\text{unsamp}}})}{\sum_{y'} \exp(z_{y'} + \Delta z_{y'})} < \pi_\theta(y_{\text{unsamp}}|x)
\end{equation}
The unsampled solution's probability decreases as the probability mass shifts toward solutions that were sampled. This creates a multiplicative dynamic: solutions that are already more probable receive larger updates because they are more probable and sampled more often. After each gradient step, the probability distribution over $\mathcal{Y}^+$ becomes more concentrated: the gap between high- and low-probability correct solutions widens. Since $\nabla_\theta \log \pi_\theta(y|x)$ points in the direction that increases $\pi_\theta(y|x)$, solutions with higher current probability receive updates of larger expected magnitude, amplifying the initial disparity.
\end{proof}

\begin{theorem}[Compounding Divergence]
\label{app:proof_divergence}
Let $\pi_t$ denote the policy after $t$ updates. Under GRPO, for two correct solutions $y_1, y_2 \in \mathcal{Y}^+(x)$ with $\pi_0(y_1|x) > \pi_0(y_2|x)$, the expected change in their log-probability ratio at step $t$ is:
\begin{equation}
\mathbb{E}\left[\Delta \log \frac{\pi_t(y_1|x)}{\pi_t(y_2|x)}\right] = \eta \cdot A^+ \cdot K \cdot \bigl(\pi_t(y_1|x) - \pi_t(y_2|x)\bigr) > 0,
\label{eq:ratio_divergence}
\end{equation}
where $\eta$ is the effective learning rate. The log-ratio is a submartingale with positive drift that increases as the probability gap widens.
\end{theorem}
\begin{proof}[Proof.]
 
Under softmax parameterization, $\pi_t(y|x) = \exp(z_y)/Z$ where $Z = \sum_{y'} \exp(z_{y'})$. The log-probability ratio is:
\[
\log \frac{\pi_t(y_1|x)}{\pi_t(y_2|x)} = z_{y_1} - z_{y_2}.
\]
 
From Theorem~\ref{thm:gradient_amplification}, the expected logit update is:
\[
\mathbb{E}[\Delta z_y] = \eta \cdot A^+ \cdot K \cdot \pi_t(y|x),
\]
where $\eta$ absorbs learning rate and gradient-to-logit conversion factors.
 
The expected change in the log-ratio is:
\begin{align}
\mathbb{E}\left[\Delta \log \frac{\pi_t(y_1|x)}{\pi_t(y_2|x)}\right] 
&= \mathbb{E}[\Delta z_{y_1}] - \mathbb{E}[\Delta z_{y_2}] \\
&= \eta \cdot A^+ \cdot K \cdot \pi_t(y_1|x) - \eta \cdot A^+ \cdot K \cdot \pi_t(y_2|x) \\
&= \eta \cdot A^+ \cdot K \cdot \bigl(\pi_t(y_1|x) - \pi_t(y_2|x)\bigr).
\end{align}
 
When $\pi_t(y_1|x) > \pi_t(y_2|x)$, this quantity is strictly positive. Moreover, as training progresses:
\begin{itemize}
    \item $\pi_t(y_1|x)$ increases (receiving larger updates)
    \item $\pi_t(y_2|x)$ decreases (receiving smaller updates, or losing mass to normalization)
    \item The gap $\pi_t(y_1|x) - \pi_t(y_2|x)$ widens
    \item The drift in \eqref{eq:ratio_divergence} increases
\end{itemize}
 
This establishes that the log-ratio diverges with accelerating speed: the collapse is not merely linear but compounds over training.
\end{proof}

\begin{corollary}[Sampling Threshold]
When $\pi_t(y|x) \ll 1/K$, the probability that solution $y$ receives at least one sample in a batch of size $K$ is:
\[
P(N_y \geq 1) = 1 - (1 - \pi_t(y|x))^K \approx K \cdot \pi_t(y|x) \ll 1
\]
Below this threshold, solution $y$ receives effectively zero gradient signal and cannot recover under on-policy sampling.
\end{corollary}
\begin{proof}[Proof.]
The probability that $y$ is never sampled in $K$ trials is:
\[
P(N_y = 0) = (1 - \pi_t(y|x))^K
\]

Thus:
\[
P(N_y \geq 1) = 1 - (1 - \pi_t(y|x))^K
\]

For small $\pi_t(y|x)$, using $(1-p)^K \approx e^{-Kp} \approx 1 - Kp$ for $Kp \ll 1$:
\[
P(N_y \geq 1) \approx K \pi_t(y|x)
\]

When $\pi_t(y|x) \ll 1/K$, we have $P(N_y \geq 1) \ll 1$, meaning $y$ is almost never sampled.

Without samples, the gradient contribution from $y$ is zero, and $\pi_t(y|x)$ can only decrease (due to the softmax normalization as other solutions increase). The expected waiting time for at least one sample is $1/\mathbb{P}(N_y \geq 1) = \Omega(1/(Kp))$. This makes the collapse irreversible.
\end{proof}

%% file: appendix/App_Optimal_policy_proofs.tex
\subsection{Proofs for uniform-correct optimality}
\label{app:optimal_policy_proofs}
\begin{theorem}[Robustness under distribution shift]
\label{app:proof_robustness}
Suppose a distribution shift disables a subset $S \subset \mathcal{Y}^{+}$ of correct outputs, where $|S| \leq s$. Consider the robust objective:
\begin{equation}
    \max_\pi \min_{S: |S| \leq s} \sum_{y \in \mathcal{Y}^{+} \setminus S} \pi(y|x).
\end{equation}
Then the uniform distribution over $\mathcal{Y}^{+}$ uniquely maximizes the worst-case retained probability mass.
\end{theorem}

\begin{proof}
Let $\pi$ be any distribution over $\mathcal{Y}^+(x)$ with $\sum_{y \in \mathcal{Y}^{+}(x)} \pi(y|x) = 1$. The adversary choosing $S$ will select the $s$ outputs with highest probability mass. 

For the uniform distribution $\pi^*(y|x) = 1/|\mathcal{Y}^{+}|$, the worst-case retained mass is:
\begin{equation}
    1 - \frac{s}{|\mathcal{Y}^{+}(x)|}.
\end{equation}

For any non-uniform distribution, let $$\pi_{\max} = \max_{y \in \mathcal{Y}^{+}} \pi(y|x) > \frac{1}{|\mathcal{Y}^{+}(x)|}.$$ The adversary can achieve retained mass strictly less than $1 - s/|\mathcal{Y}^{+}(x)|$ by targeting high-probability outputs.

Sorting probabilities in decreasing order: $$\pi(y_{(1)}|x) \geq \pi(y_{(2)}|x) \geq \ldots$$ the adversary removes $\{y_{(1)}, \ldots, y_{(s)}\}$, achieving retained mass:
\begin{equation}
    \sum_{j > s} \pi(y_{(j)}|x) = 1 - \sum_{j=1}^s \pi(y_{(j)}|x) \leq 1 - \frac{s}{|\mathcal{Y}^{+}(x)|},
\end{equation}
with equality if and only if $\pi$ is uniform. Thus, the uniform distribution over $\mathcal{Y}^{+}(x)$ uniquely maximizes the worst-case retained probability mass.
\end{proof}

\begin{theorem}
\label{app:proof_entropy_regularized}
Consider the entropy-regularized objective
\begin{equation}
\max_{\pi(\cdot \mid x)} \;
\mathbb{E}_{y \sim \pi}[R(x,y)]
+ \tau \, \mathrm{H}(\pi(\cdot \mid x)).
\end{equation}
There exists a threshold $\tau_0 > 0$ such that for all $0 < \tau < \tau_0$, any optimal policy $\pi^*$ places zero mass on incorrect outputs and is uniform over $\mathcal{Y}^+(x)$.
\end{theorem}

% \blcomment{A bit hard to understand the purpose of this theorem}

\begin{proof}
Let $R(x,y) \in \{0,1\}$ and  $\mathcal{Y}^+(x) = \{y : R(x,y)=1\}$ is non-empty and not equal to the full output space $\mathcal{Y}$. Let $\mathcal{Y}^- = \mathcal{Y} \setminus \mathcal{Y}^+(x)$ denote incorrect outputs. Taking the objective as
\[
J(\pi) = Z(x) + \tau H(\pi),
\qquad
Z(x) = \sum_{y \in \mathcal{Y}^+} \pi(y \mid x).
\]

\paragraph{Step 1: Support concentrates on correct outputs for small $\tau$.}
Suppose $\pi$ assigns mass $\epsilon > 0$ to $\mathcal{Y}^-$. Construct a new policy $\tilde{\pi}$ by moving this mass to any correct output $y^+ \in \mathcal{Y}^+$.

The reward change is
\[
\Delta Z = Z(\tilde{\pi}) - Z(\pi) = \epsilon.
\]
Moving probability mass between two outcomes changes entropy by at most $O(\epsilon \log |\mathcal{Y}|)$, so
\[
J(\tilde{\pi}) - J(\pi)
\ge \epsilon - \tau C \epsilon
\]
for some constant $C \le \log |\mathcal{Y}|$. If $\tau < 1/C$, transferring mass from incorrect to correct strictly improves the objective. Therefore any maximizer must satisfy $Z(x)=1$, i.e., $\pi(y \mid x)=0$ for all $y \in \mathcal{Y}^-$.

\paragraph{Step 2: Uniformity within the correct set.}
Given $Z(x)=1$, the objective reduces to
\[
J(\pi) = 1 + \tau H(\pi),
\]
with $\pi$ supported on $\mathcal{Y}^+(x)$. Let $q(y \mid x) = \pi(y \mid x)$ for $y \in \mathcal{Y}^+(x)$. Then
\[
\max_q H(q)
\quad \text{s.t.} \quad
\sum_{y \in \mathcal{Y}^+(x)} q(y \mid x) = 1.
\]
Entropy is strictly concave, so the unique maximizer is the uniform distribution:
\[
q^*(y \mid x) = \frac{1}{|\mathcal{Y}^+(x)|}.
\]
Thus, for sufficiently small $\tau$, the optimal entropy-regularized policy is uniform over correct outputs.
\end{proof}

%% file: appendix/App_UCPO_proofs.tex
\subsection{Proof of Gradient Decomposition Proposition ~\ref{prop:grad-decomp} }
\label{app:grad-proof}

\begin{proposition}[Gradient Decomposition, restated]
Let $x$ be a prompt with $\mathcal{Y}^+(x) \neq \emptyset$.
Define the correctness mass $Z_\theta(x) = \sum_{y \in \mathcal{Y}^+(x)} \pi_\theta(y \mid x)$, the conditional distribution $q_\theta(y \mid x) = \pi_\theta(y \mid x) / Z_\theta(x)$, and the uniform distribution $u(y \mid x) = 1/|\mathcal{Y}^+(x)|$.
Then the gradient of the UCPO objective (excluding the reference KL term) is:
\[
\nabla_\theta L_{\mathrm{UCPO}}(x) = \sum_{y \in \mathcal{Y}^+(x)} \Big[ (1-\tau)\, q_\theta(y \mid x) + \tau\, u(y \mid x) \Big] \nabla_\theta \log \pi_\theta(y \mid x).
\]
\end{proposition}

\begin{proof}
We compute the gradients of each term in $L_{\mathrm{UCPO}}(x) = \log Z_\theta(x) - \tau \, \mathrm{KL}(u \| q_\theta)$ separately.

\paragraph{Step 1: Gradient of $\log Z_\theta(x)$.}
By the chain rule,
\begin{align}
\nabla_\theta \log Z_\theta(x) 
&= \frac{1}{Z_\theta(x)} \nabla_\theta Z_\theta(x) \\
&= \frac{1}{Z_\theta(x)} \sum_{y \in \mathcal{Y}^+(x)} \nabla_\theta \pi_\theta(y \mid x) \\
&= \frac{1}{Z_\theta(x)} \sum_{y \in \mathcal{Y}^+(x)} \pi_\theta(y \mid x) \nabla_\theta \log \pi_\theta(y \mid x) \\
&= \sum_{y \in \mathcal{Y}^+(x)} q_\theta(y \mid x) \nabla_\theta \log \pi_\theta(y \mid x).
\end{align}

\paragraph{Step 2: Gradient of $\mathrm{KL}(u \| q_\theta)$.}
Expanding the KL divergence,
\begin{align}
\mathrm{KL}(u \| q_\theta) 
&= \sum_{y \in \mathcal{Y}^+(x)} u(y \mid x) \log \frac{u(y \mid x)}{q_\theta(y \mid x)} \\
&= \sum_{y \in \mathcal{Y}^+(x)} u(y \mid x) \log u(y \mid x) - \sum_{y \in \mathcal{Y}^+(x)} u(y \mid x) \log q_\theta(y \mid x).
\end{align}
The first sum is a constant (entropy of $u$), so
\[
\nabla_\theta \mathrm{KL}(u \| q_\theta) = -\sum_{y \in \mathcal{Y}^+(x)} u(y \mid x) \nabla_\theta \log q_\theta(y \mid x).
\]

\paragraph{Step 3: Gradient of $\log q_\theta(y \mid x)$.}
Since $q_\theta(y \mid x) = \pi_\theta(y \mid x) / Z_\theta(x)$,
\begin{align}
\nabla_\theta \log q_\theta(y \mid x) 
&= \nabla_\theta \log \pi_\theta(y \mid x) - \nabla_\theta \log Z_\theta(x).
\end{align}

\paragraph{Step 4: Combining.}
Substituting Step 3 into Step 2:
\begin{align}
\nabla_\theta \mathrm{KL}(u \| q_\theta) 
&= -\sum_{y} u(y) \Big[ \nabla_\theta \log \pi_\theta(y) - \nabla_\theta \log Z_\theta(x) \Big] \\
&= -\sum_{y} u(y) \nabla_\theta \log \pi_\theta(y) + \nabla_\theta \log Z_\theta(x),
\end{align}
where we used $\sum_y u(y) = 1$.

Therefore,
\begin{align}
\nabla_\theta \big[ -\tau \, \mathrm{KL}(u \| q_\theta) \big] 
&= \tau \sum_{y} u(y) \nabla_\theta \log \pi_\theta(y) - \tau \nabla_\theta \log Z_\theta(x).
\end{align}

\paragraph{Step 5: Full gradient.}
Adding the gradients from Steps 1 and 4:
\begin{align}
\nabla_\theta L_{\mathrm{UCPO}}(x) 
&= \nabla_\theta \log Z_\theta(x) + \tau \sum_{y} u(y) \nabla_\theta \log \pi_\theta(y) - \tau \nabla_\theta \log Z_\theta(x) \\
&= (1-\tau) \nabla_\theta \log Z_\theta(x) + \tau \sum_{y} u(y) \nabla_\theta \log \pi_\theta(y) \\
&= (1-\tau) \sum_{y} q_\theta(y) \nabla_\theta \log \pi_\theta(y) + \tau \sum_{y} u(y) \nabla_\theta \log \pi_\theta(y) \\
&= \sum_{y \in \mathcal{Y}^+(x)} \Big[ (1-\tau) q_\theta(y) + \tau u(y) \Big] \nabla_\theta \log \pi_\theta(y). \qedhere
\end{align}
\end{proof}

\subsection{Derivation of the Inverse-q Importance Weights}
\label{app:snis}
We know the gradient weights each correct response by a convex combination of its current conditional probability $q_\theta(y)$ and the uniform weight $u(y)$.
Rewriting as expectations:
\begin{equation}
\label{eq:grad-expectations}
\nabla_\theta L_{\mathrm{UCPO}} = (1-\tau)\, \mathbb{E}_{y \sim q_\theta}[\nabla_\theta \log \pi_\theta(y \mid x)] + \tau\, \mathbb{E}_{y \sim u}[\nabla_\theta \log \pi_\theta(y \mid x)].
\end{equation}
We estimate the terms from the $n$ correct rollouts $\{y_i\}_{i=1}^n \sim q_\theta(\cdot \mid x)$ out of the sampled N rollouts from $\pi_\theta(\cdot \mid x)$ .
The first term is an expectation under $q_\theta$, estimated directly as a sample mean:
\begin{equation}
\mathbb{E}_{y \sim q_\theta}[\nabla_\theta \log \pi_\theta(y)] \approx \frac{1}{n} \sum_{i=1}^n \nabla_\theta \log \pi_\theta(y_i \mid x).
\end{equation}

The second term is an expectation under the uniform distribution $u$, but we only have samples from $q_\theta$. So, we derive the importance sampling estimator for $\mathbb{E}_{y \sim u}[f(y)]$ using samples from $q_\theta$.

\paragraph{Importance sampling identity.}
For any function $f$,
\begin{equation}
\mathbb{E}_{y \sim u}[f(y)] = \mathbb{E}_{y \sim q_\theta}\left[ \frac{u(y)}{q_\theta(y)} f(y) \right].
\end{equation}

\paragraph{Simplifying the importance ratio.}
Since $u(y) = 1/|\mathcal{Y}^+(x)|$ is constant over the correct set,
\begin{equation}
\frac{u(y)}{q_\theta(y)} = \frac{1}{|\mathcal{Y}^+(x)|} \cdot \frac{1}{q_\theta(y)}.
\end{equation}
The factor $1/|\mathcal{Y}^+(x)|$ is independent of $y$ and cancels under normalization.
Thus, the importance weight is proportional to $1/q_\theta(y)$.
Given samples $\{y_i\}_{i=1}^n \sim q_\theta$, the unnormalized importance sampling estimator is
\[
\mathbb{E}_{y \sim u}[f(y)] \approx \frac{1}{n \cdot |\mathcal{Y}^+(x)|} \sum_{i=1}^n \frac{1}{q_\theta(y_i)} f(y_i).
\]

However, we do not know $|\mathcal{Y}^+(x)|$ or $q_\theta(y_i)$ exactly.

\paragraph{Self-normalized estimator.}
Given samples $\{y_i\}_{i=1}^n \sim q_\theta$, we estimate $q_\theta(y_i)$ empirically:
\begin{equation}
\hat{q}_i = \frac{\pi_\theta(y_i \mid x)}{\sum_{j=1}^n \pi_\theta(y_j \mid x)}.
\end{equation}
Define unnormalized weights $v_i = 1/\hat{q}_i$.
The self-normalized importance sampling estimator is:
\begin{equation}
\mathbb{E}_{y \sim u}[f(y)] \approx \sum_{i=1}^n \frac{v_i}{\sum_{j=1}^n v_j} f(y_i).
\end{equation}
\begin{equation}
\mathbb{E}_{y \sim u}[\nabla_\theta \log \pi_\theta(y)] \approx \sum_{i=1}^n \frac{v_i}{\sum_{j=1}^n v_j} \nabla_\theta \log \pi_\theta(y_i \mid x).
\end{equation}
Self-normalization avoids needing to know $|\mathcal{Y}^+(x)|$ and ensures bounded weights.
It introduces $O(1/n)$ bias but significantly reduces variance compared to unnormalized importance sampling, which is critical when $\hat{q}_i$ may be small for rare correct responses.

Combining both terms yields the UCPO gradient estimator:
\begin{equation}
\label{eq:ucpo-estimator}
\nabla_\theta L_{\mathrm{UCPO}} \approx \sum_{i=1}^n w_i \, \nabla_\theta \log \pi_\theta(y_i \mid x),
\end{equation}
where
\begin{equation}
\label{eq:blended-weights}
w_i = \frac{1-\tau}{n} + \tau \cdot \frac{v_i}{\sum_{j=1}^n v_j}.
\end{equation}

Responses with high $q_\theta(y_i)$ (already favored by the policy) receive low weight $v_i$.
Responses with low $q_\theta(y_i)$ (rare under current policy) receive high weight.
This reweighting transforms samples from $q_\theta$ into an approximation of uniform samples over the correct set.

\subsection{Proofs for Proposition ~\ref{prop:mass} and Theorem ~\ref{thm:stationary}}
\label{app:proofs}

\begin{proposition}[Mass Preservation, restated]
Let $A_{\mathrm{base}}$ be the GRPO advantage for correct responses, and let $n$ be the number of correct rollouts.
Define UCPO advantages as $A_i^{\mathrm{UCPO}} = A_{\mathrm{base}} \cdot n \cdot w_i^{\mathrm{blend}}$ where
\[
w_i^{\mathrm{blend}} = \frac{1-\tau}{n} + \tau \cdot \frac{v_i}{\sum_{j=1}^n v_j}.
\]
Then:
\[
\sum_{i=1}^n A_i^{\mathrm{UCPO}} = \sum_{i=1}^n A_i^{\mathrm{GRPO}} = n \cdot A_{\mathrm{base}}.
\]
\end{proposition}

\begin{proof}
We verify that the blended weights sum to one:
\begin{align}
\sum_{i=1}^n w_i^{\mathrm{blend}} 
&= \sum_{i=1}^n \left( \frac{1-\tau}{n} + \tau \cdot \frac{v_i}{\sum_{j=1}^n v_j} \right) \\
&= (1-\tau) \cdot \frac{n}{n} + \tau \cdot \frac{\sum_{i=1}^n v_i}{\sum_{j=1}^n v_j} \\
&= (1-\tau) + \tau \\
&= 1.
\end{align}

Therefore,
\[
\sum_{i=1}^n A_i^{\mathrm{UCPO}} = A_{\mathrm{base}} \cdot n \cdot \sum_{i=1}^n w_i^{\mathrm{blend}} = A_{\mathrm{base}} \cdot n \cdot 1 = n \cdot A_{\mathrm{base}}.
\]

Standard GRPO assigns $A_{\mathrm{base}}$ to each of the $n$ correct responses, so $\sum_i A_i^{\mathrm{GRPO}} = n \cdot A_{\mathrm{base}}$.
\end{proof}

\begin{theorem}[Uniform-correct policy uniquely maximizes the UCPO objective]
\label{thm:ucpo-unique-opt}
For a fixed prompt $x$, consider the UCPO objective without the reference KL term:
\[
L_{\mathrm{UCPO}}(x)
=
\log Z_\theta(x)
-
\tau\,\mathrm{KL}\!\bigl(u(\cdot\mid x)\,\|\,q_\theta(\cdot\mid x)\bigr),
\]
where
\[
Z_\theta(x)=\sum_{y\in \mathcal{Y}^+(x)} \pi_\theta(y\mid x),
\qquad
q_\theta(y\mid x)=\frac{\pi_\theta(y\mid x)}{Z_\theta(x)}
\quad \text{for } y\in\mathcal{Y}^+(x),
\]
and
\[
u(y\mid x)=\frac{1}{|\mathcal{Y}^+(x)|}
\quad \text{for } y\in\mathcal{Y}^+(x).
\]
For any $\tau>0$, the unique maximizer of $L_{\mathrm{UCPO}}(x)$ over policies $\pi_\theta(\cdot\mid x)$ is the uniform-correct policy
\[
\pi^*(y\mid x)=
\begin{cases}
\frac{1}{|\mathcal{Y}^+(x)|}, & y\in\mathcal{Y}^+(x),\\[4pt]
0, & y\notin\mathcal{Y}^+(x).
\end{cases}
\]
\end{theorem}

\begin{proof}
For any policy $\pi_\theta(\cdot\mid x)$, since $\pi_\theta(\cdot\mid x)$ is a probability distribution over the full output space,
\[
Z_\theta(x)=\sum_{y\in\mathcal{Y}^+(x)} \pi_\theta(y\mid x)\le 1.
\]
Therefore,
\[
\log Z_\theta(x)\le 0,
\]
with equality if and only if $Z_\theta(x)=1$, i.e., if and only if
\[
\pi_\theta(y\mid x)=0
\qquad \text{for all } y\notin\mathcal{Y}^+(x).
\]

Next, by non-negativity of KL divergence,
\[
\mathrm{KL}\!\bigl(u(\cdot\mid x)\,\|\,q_\theta(\cdot\mid x)\bigr)\ge 0,
\]
so
\[
-\tau\,\mathrm{KL}\!\bigl(u(\cdot\mid x)\,\|\,q_\theta(\cdot\mid x)\bigr)\le 0,
\]
with equality if and only if
\[
q_\theta(\cdot\mid x)=u(\cdot\mid x).
\]

Combining the two inequalities,
\[
L_{\mathrm{UCPO}}(x)
=
\log Z_\theta(x)
-
\tau\,\mathrm{KL}\!\bigl(u(\cdot\mid x)\,\|\,q_\theta(\cdot\mid x)\bigr)
\le 0.
\]
Equality holds if and only if both conditions are satisfied:
\begin{enumerate}
    \item $Z_\theta(x)=1$, so all probability mass lies on $\mathcal{Y}^+(x)$; and
    \item $q_\theta(\cdot\mid x)=u(\cdot\mid x)$, so the conditional distribution over correct responses is uniform.
\end{enumerate}

These two conditions uniquely determine
\[
\pi_\theta(y\mid x)=
\begin{cases}
\frac{1}{|\mathcal{Y}^+(x)|}, & y\in\mathcal{Y}^+(x),\\[4pt]
0, & y\notin\mathcal{Y}^+(x),
\end{cases}
\]
which is exactly the uniform-correct policy $\pi^*$.
Hence $\pi^*$ is the unique maximizer of $L_{\mathrm{UCPO}}(x)$.
\end{proof}

%% file: appendix/choice_of_divergence.tex
\subsection{Choice of Divergence}
\label{app:divergence-choice}

The UCPO objective penalizes deviation from the uniform target using $\mathrm{KL}(u \,\|\, q_\theta)$ rather than $\mathrm{KL}(q_\theta \,\|\, u)$. 
While both divergences are minimized when $q_\theta = u$, their gradients differ in a way that is critical for escaping mode collapse.
We derive the gradient of each and explain why $\mathrm{KL}(u \| q_\theta)$ is the appropriate choice.

\subsubsection{Entropy Maximization: $\mathrm{KL}(q_\theta \,\|\, u)$}

When $u$ is uniform over $\mathcal{Y}^+(x)$, we have $\log u(y) = -\log|\mathcal{Y}^+(x)|$ for all $y \in \mathcal{Y}^+(x)$.
Expanding the KL divergence:
\begin{align}
\mathrm{KL}(q_\theta \| u) 
&= \sum_{y \in \mathcal{Y}^+(x)} q_\theta(y) \log \frac{q_\theta(y)}{u(y)} \\
&= \sum_{y} q_\theta(y) \log q_\theta(y) - \sum_{y} q_\theta(y) \log u(y) \\
&= -H(q_\theta) + \log|\mathcal{Y}^+(x)|.
\end{align}
Thus, minimizing $\mathrm{KL}(q_\theta \| u)$ is equivalent to maximizing the entropy $H(q_\theta)$.

The entropy of $q_\theta$ is:
\[
H(q_\theta) = -\sum_{y \in \mathcal{Y}^+(x)} q_\theta(y) \log q_\theta(y).
\]
Taking the gradient with respect to $\theta$:
\begin{align}
\nabla_\theta H(q_\theta) 
&= -\sum_{y} \nabla_\theta \big[ q_\theta(y) \log q_\theta(y) \big] \\
&= -\sum_{y} \big[ \nabla_\theta q_\theta(y) \cdot \log q_\theta(y) + q_\theta(y) \cdot \nabla_\theta \log q_\theta(y) \big] \\
&= -\sum_{y} \nabla_\theta q_\theta(y) \cdot \log q_\theta(y) - \sum_{y} q_\theta(y) \nabla_\theta \log q_\theta(y).
\end{align}
Since $\sum_y q_\theta(y) = 1$, we have $\sum_y \nabla_\theta q_\theta(y) = 0$.
Using $\nabla_\theta q_\theta(y) = q_\theta(y) \nabla_\theta \log q_\theta(y)$:
\begin{align}
\nabla_\theta H(q_\theta) 
&= -\sum_{y} q_\theta(y) \nabla_\theta \log q_\theta(y) \cdot \log q_\theta(y) - \sum_{y} q_\theta(y) \nabla_\theta \log q_\theta(y) \\
&= -\sum_{y} q_\theta(y) \big[ 1 + \log q_\theta(y) \big] \nabla_\theta \log q_\theta(y) \\
&= -\mathbb{E}_{y \sim q_\theta} \big[ (1 + \log q_\theta(y)) \nabla_\theta \log q_\theta(y) \big].
\end{align}
The gradient is weighted by the current distribution $q_\theta$. High-probability responses dominate the gradient; low-probability responses contribute negligibly.
This preserves the feedback loop identified in Section~\ref{sec:collapse}: frequently sampled solutions receive larger updates, reinforcing their dominance.

\subsubsection{Uniform-Weighted Divergence: $\mathrm{KL}(u \,\|\, q_\theta)$}

Expanding the KL divergence:
\begin{align}
\mathrm{KL}(u \| q_\theta) 
&= \sum_{y \in \mathcal{Y}^+(x)} u(y) \log \frac{u(y)}{q_\theta(y)} \\
&= \sum_{y} u(y) \log u(y) - \sum_{y} u(y) \log q_\theta(y) \\
&= -H(u) - \sum_{y} u(y) \log q_\theta(y).
\end{align}
Since $H(u)$ is constant with respect to $\theta$:
\[
-\mathrm{KL}(u \| q_\theta) = \sum_{y} u(y) \log q_\theta(y) + \mathrm{const}.
\]

Substituting $q_\theta(y) = \pi_\theta(y) / Z_\theta(x)$ where $Z_\theta(x) = \sum_{y \in \mathcal{Y}^+(x)} \pi_\theta(y)$:
\begin{align}
-\mathrm{KL}(u \| q_\theta) 
&= \sum_{y} u(y) \big[ \log \pi_\theta(y) - \log Z_\theta(x) \big] + \mathrm{const} \\
&= \sum_{y} u(y) \log \pi_\theta(y) - \log Z_\theta(x) + \mathrm{const}.
\end{align}
Taking the gradient:
\begin{align}
\nabla_\theta \big[ -\mathrm{KL}(u \| q_\theta) \big]
&= \sum_{y} u(y) \nabla_\theta \log \pi_\theta(y) - \nabla_\theta \log Z_\theta(x).
\end{align}

For the second term, we use:
\begin{align}
\nabla_\theta \log Z_\theta(x) 
&= \frac{1}{Z_\theta(x)} \sum_{y \in \mathcal{Y}^+(x)} \nabla_\theta \pi_\theta(y) \\
&= \sum_{y} \frac{\pi_\theta(y)}{Z_\theta(x)} \nabla_\theta \log \pi_\theta(y) \\
&= \sum_{y} q_\theta(y) \nabla_\theta \log \pi_\theta(y).
\end{align}

Substituting:
\begin{equation}
\nabla_\theta \big[ -\mathrm{KL}(u \| q_\theta) \big]
= \mathbb{E}_{y \sim u} \big[ \nabla_\theta \log \pi_\theta(y) \big] - \mathbb{E}_{y \sim q_\theta} \big[ \nabla_\theta \log \pi_\theta(y) \big].
\end{equation}

The gradient consists of two terms:
\begin{itemize}
    \item $\mathbb{E}_{y \sim u}[\nabla_\theta \log \pi_\theta(y)]$: Increases probability of all correct responses equally, regardless of their current probability.
    \item $-\mathbb{E}_{y \sim q_\theta}[\nabla_\theta \log \pi_\theta(y)]$: Decreases probability of responses in proportion to their current dominance.
\end{itemize}
Together, these terms redistribute gradient signal from overrepresented to underrepresented responses, breaking the self-reinforcing feedback loop.

\subsubsection{Connection to UCPO}
Combining $\mathrm{KL}(u \| q_\theta)$ with the correctness objective $\log Z_\theta(x)$ yields the UCPO gradient (Proposition~\ref{prop:grad-decomp}):
\[
\nabla_\theta L_{\mathrm{UCPO}} 
= \sum_{y \in \mathcal{Y}^+(x)} 
\big[(1-\tau)\, q_\theta(y) + \tau\, u(y)\big]\, \nabla_\theta \log \pi_\theta(y \mid x).
\]
At $\tau = 0$, updates are fully $q_\theta$-weighted (GRPO).
As $\tau$ increases, updates shift toward uniform weighting, directly counteracting the sampling bias that drives collapse.

%% file: appendix/empirical_div_collapse.tex
\label{app:toy_exp_collpase}
Section ~\ref{sec:collapse} establishes that diversity collapse under RLVR arises from the interaction between objective indifference and optimization dynamics, which together induce a self-reinforcing concentration of probability mass on a subset of correct solutions. While this mechanism is derived analytically, its empirical manifestation in large-scale language models is difficult to observe directly due to partial observability and sequence-level complexity.

To isolate and validate the proposed mechanism, we construct a controlled RLVR environment in which the policy is fully observable and all quantities can be computed exactly. The goal of this experiment is to causally verify each component of the collapse mechanism under minimal assumptions. 

\paragraph{Setup.}
We define a discrete output space of 20 tokens, among which a subset of 3 tokens are designated as correct:
\[
\mathcal{Y}^{+} = \{y_1, y_2, y_3\}, 
\qquad 
\mathcal{Y}^{-} = \mathcal{Y} \setminus \mathcal{Y}^{+}.
\]

The policy $\pi_{\theta}(y)$ is parameterized as a softmax over logits $z_y$:
\[
\pi_{\theta}(y) = \frac{\exp(z_y)}{\sum_{y'} \exp(z_{y'})}.
\]

To reflect realistic pretrained biases, the initial policy assigns non-uniform probability mass within the correct set, while distributing the remaining mass across incorrect outputs. This ensures that symmetry is broken at initialization, allowing us to study how small imbalances evolve under training.

\paragraph{Training Dynamics.}
At each iteration:
\begin{itemize}
    \item We sample $K$ outputs from $\pi_{\theta}$ (on-policy sampling).
    \item Rewards are assigned deterministically:
    \[
    R(y) = \mathbf{1}[y \in \mathcal{Y}^{+}].
    \]
    \item The policy is updated using either GRPO, UCPO, and Global entropy regularization.
\end{itemize}

\paragraph{Tracked Metrics.}
Since the full distribution is explicitly available, we analyze training dynamics through the following quantities:

\textbf{(i) Sampling statistics.}  
The number of times each output is sampled: $N_y$.

\textbf{(ii) Conditional distribution over correct solutions.}  
We define the normalized distribution restricted to correct outputs:
\[
q_\theta(y \mid \mathcal{Y}^{+}) = \frac{\pi_{\theta}(y)}{\sum_{y' \in \mathcal{Y}^{+}} \pi_{\theta}(y')}, 
\quad y \in \mathcal{Y}^{+}.
\]

\textbf{(iii) Entropy of correct solutions.}  
To measure diversity within the correct set:
\[
H(q_\theta) = - \sum_{y \in \mathcal{Y}^{+}} q_\theta(y \mid \mathcal{Y}^{+}) \log q_\theta(y \mid \mathcal{Y}^{+}).
\]

\textbf{(iv) Total correctness mass.}  
The probability assigned to correct outputs:
\[
Z_\theta = \sum_{y \in \mathcal{Y}^{+}} \pi_{\theta}(y).
\]

\textbf{(v) Relative dominance.}  
We monitor pairwise probability ratios between correct solutions to capture mode concentration:
\[
\frac{\pi_{\theta}(y_i)}{\pi_{\theta}(y_j)}, \quad y_i, y_j \in \mathcal{Y}^{+}.
\]

\subsection{GRPO Induces Self-Reinforcing Collapse}
\label{app:finding_1_collpase}
We  validate each step of the collapse mechanism (Section 5), with the corresponding evidence shown in Figure ~\ref{fig:collapse_figure}. 
\begin{enumerate}
    \item \textbf{Sampling bias: } Figure ~\ref{fig:collapse_figure}(a) shows that higher-probability solutions are sampled more frequently, with empirical counts closely matching theoretical expectation:
\[
\mathbb{E}[N_y] = K \cdot \pi_{\theta}(y).
\]

Even under near-uniform initialization (Appendix ~\ref{app:finding_2_always_skewed}), stochastic sampling produces unequal counts, confirming the inherent bias of on-policy estimation.
    \item \textbf{Gradient amplification:} Figure ~\ref{fig:collapse_figure}(b) plots $f(p) = p(1 - p)$, the function proportional to the expected gradient magnitude for a softmax policy. The green dots mark the initial probabilities of the three correct tokens. The key observation is that tokens already at higher probability receive a larger expected gradient update. Under GRPO, all correct samples receive the same normalized advantage $A^{+}$, so the total expected update to $\log \pi(y)$ scales as:
\[
\Delta \log \pi_\theta(y) \propto \eta \cdot A^{+} \cdot K \cdot \pi_\theta(y).
\]

This confirms that GRPO amplifies sampling asymmetries rather than correcting them.
    \item \textbf{Probability ratios diverge monotonically.} Figure ~\ref{fig:collapse_figure}(c) tracks the evolution of probability ratios between correct solutions. The log-ratio between dominant and minority solutions increases steadily over training:
\[
\log \frac{\pi_{\theta}(y_1)}{\pi_{\theta}(y_2)} \uparrow.
\]

This demonstrates the positive feedback loop predicted in Theorem ~\ref{thm:divergence}, that the expected change in log-ratio is proportional to the difference in probabilities, which is always positive when $y_0 \geq y_1$.
    \item \textbf{Collapse is irreversible once below the sampling threshold.} Figure ~\ref{fig:collapse_figure}(d) shows the conditional distribution
\[
q_\theta(y \mid \text{correct}) = \frac{\pi_{\theta}(y)}{Z_\theta}
\]
at five snapshots during training (steps $0, 50, 100, 200, 299$). The blue dashed line marks the uniform target $1/3$. By step $299$, virtually all mass has migrated to $y_0$. The solutions $y_1$ and $y_2$ fall so far below the sampling threshold ($\mathbb{E}[N_y] \ll 1$) that they receive negligible gradient signal and can no longer recover. This confirms the irreversibility predicted in Corollary ~\ref{cor:threshold}.
\end{enumerate}

\begin{figure*}[!t]
        \centering
        \includegraphics[width=\linewidth]{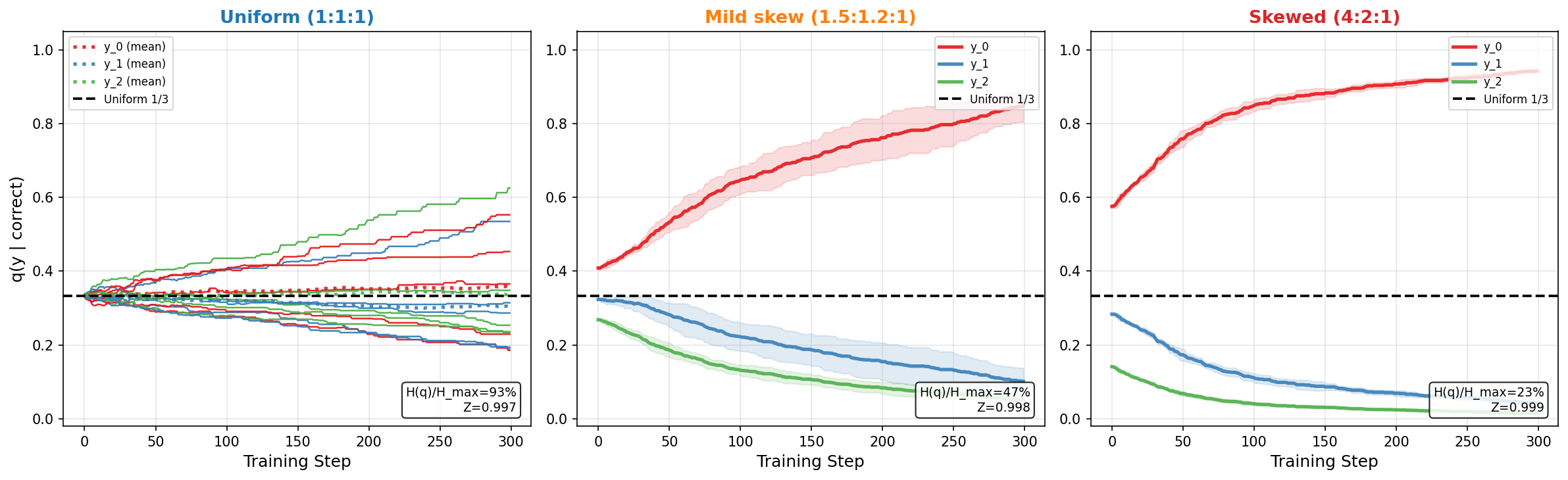}
        \vspace{-1em}
   \caption{Collapse dynamics under GRPO in a controlled RLVR environment with three correct tokens and three initialization profiles. Each panel tracks the conditional distribution $q(y \mid \mathcal{Y}^+)$ over 300 training steps across 5 random seeds; the dashed blue line marks the uniform target $1/3$. Even under uniform initialization, stochastic sampling breaks symmetry and GRPO collapses to a single dominant token. With increasing initialization skew, collapse becomes faster and more deterministic. In all cases, total correctness mass $Z$ remains high, showing that GRPO improves correctness while collapsing within-correct diversity.}
    \label{fig:init_robustness}
\end{figure*}

\subsection{Collapse Is Not Limited to Skewed Initializations}
\label{app:finding_2_always_skewed}

In Figure~\ref{fig:collapse_figure}, we use a skewed initialization of $4:2:1$ over
correct tokens to make the collapse dynamics visually clear. However, even under near-uniform initialization, stochastic sampling produces unequal counts, confirming the inherent bias of on-policy estimation, consistent with Proposition \ref{prop:sampling}. To verify this we evaluate three initialization profiles spanning the range from maximally symmetric to strongly biased. The reference policy assigns background mass $0.01$ to each incorrect token (before normalization) across all profiles:

\begin{itemize}
    \item \textbf{Uniform (1:1:1):} $q_\theta(y_0) = q_\theta(y_1) = q_\theta(y_2) = 0.333$. Tests the null
    hypothesis that collapse requires an initial imbalance.
    \item \textbf{Mild skew (1.5:1.2:1):} $q_\theta(y_0) = 0.41,\; q_\theta(y_1) = 0.32,\; q_\theta(y_2) =
    0.27$. Reflects the small asymmetry.
    \item \textbf{Skewed (4:2:1):} $q_\theta(y_0) = 0.57,\; q_\theta(y_1) = 0.29,\; q_\theta(y_2) = 0.14$.
\end{itemize}

We run 5 independent seeds per profile and track the conditional entropy $H(q_\theta)$ and
correctness mass $Z_\theta$ throughout training. Results are shown in Figure~\ref{fig:init_robustness}.
Under skewed and mild-skew initialization, the dominant token $y_0$ wins deterministically across all 5 seeds: GRPO always amplifies the initially more-probable correct token, confirming the collapse mechanism. Final conditional entropy is
$H(q_\theta)/H_\mathrm{max} \approx 16\%$ (skewed) and $32\%$ (mild skew), both far below
the uniform target.

Under uniform initialization, collapse still occurs in every seed, but the winner is arbitrary: across 5 seeds, $y_0$ collapses to the dominant mode in 3 runs, $y_1$ in 1 run, and $y_2$ in 1 run. This is a direct consequence of Proposition~\ref{prop:symmetry} (stochastic symmetry breaking): even when $\pi_0(y \mid x) = 1/|Y^+|$, finite-batch sampling produces unequal counts $\{N_y\}$ with probability 1, seeding the positive feedback loop. GRPO then amplifies whichever token happened to be over-represented in the first batch.

These results establish that mode collapse is not an artifact of a skewed initialization, it is an intrinsic consequence of the interaction between on-policy sampling variance and GRPO's indifference to within-correct diversity. The skewed profile collapses faster and more predictably; the uniform profile collapses more slowly and to an arbitrary mode.

 \begin{figure*}[!ht]
        \centering
        \includegraphics[width=\linewidth]{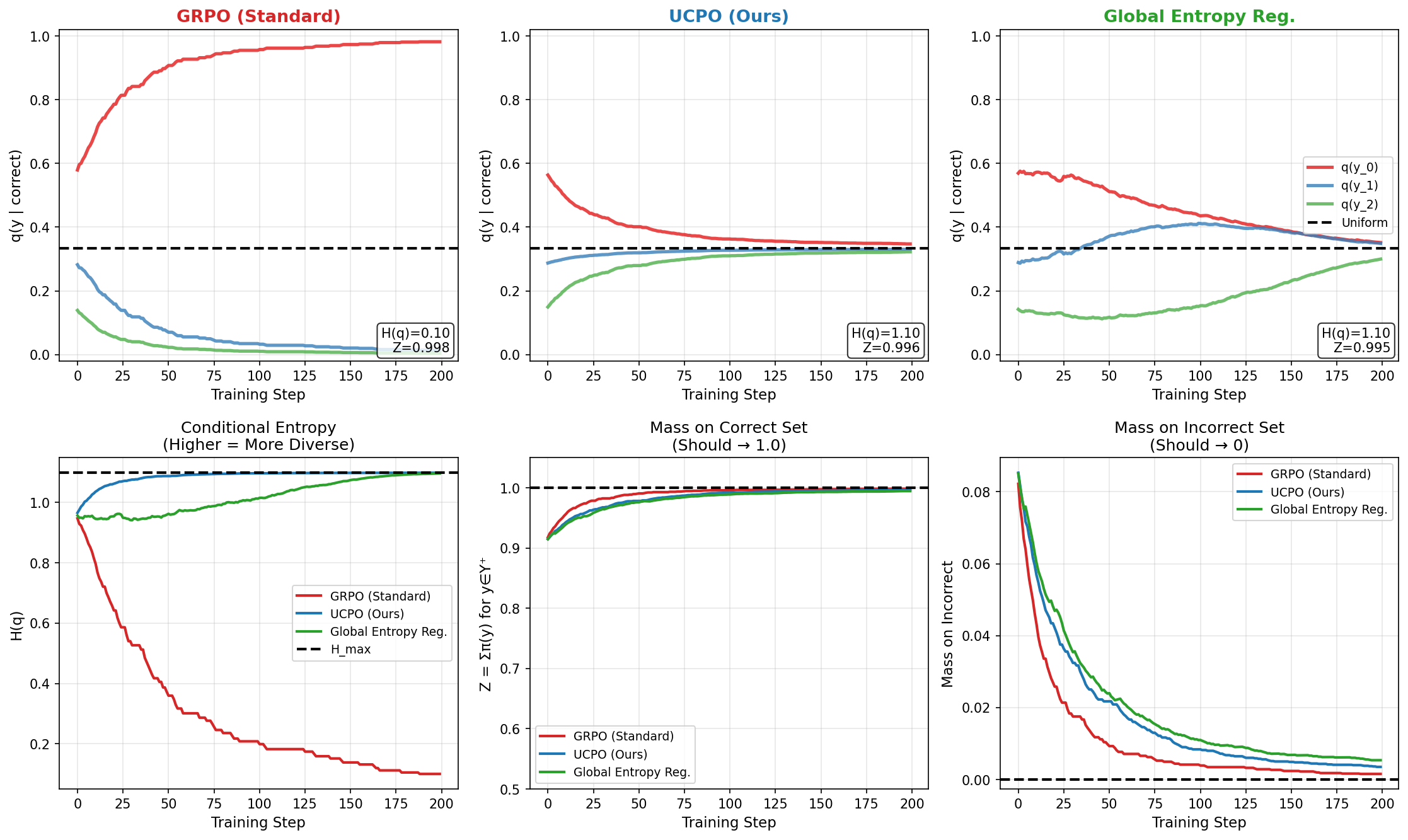}
    \caption{Training dynamics under GRPO, UCPO, and global entropy regularization.
GRPO drives the conditional distribution over correct tokens toward a single dominant mode, causing conditional entropy to collapse. UCPO instead maintains a near-uniform distribution within the correct set while preserving total correctness mass. Global entropy regularization increases entropy only by assigning probability mass to incorrect tokens, demonstrating that global entropy acts on the wrong support. }
    \label{fig:method_comp}
\end{figure*}

\subsubsection*{GRPO Induces Concentration within the Correct Set}

We further analyze the training dynamics of GRPO, UCPO, and global entropy regularization in depth (Figure~\ref{fig:method_comp}).
As shown in Figure~\ref{fig:method_comp}(a), GRPO drives the conditional distribution $q_\theta(y \mid \mathcal{Y}^{+})$ toward a single dominant solution. Correspondingly, Figure~\ref{fig:method_comp}(d) shows rapid decay in conditional entropy $H(q_\theta)$, while Figure~\ref{fig:method_comp}(e) shows that total correctness mass $Z$ remains high.

This indicates that GRPO improves correctness by concentrating probability mass rather than maintaining coverage over multiple valid solutions while UCPO and Global entropy regularization maintain a balance distribution. As a result, diversity within the correct set collapses despite high sampling efficiency.

\begin{figure*}[!ht]
        \centering
        \includegraphics[width=\linewidth]{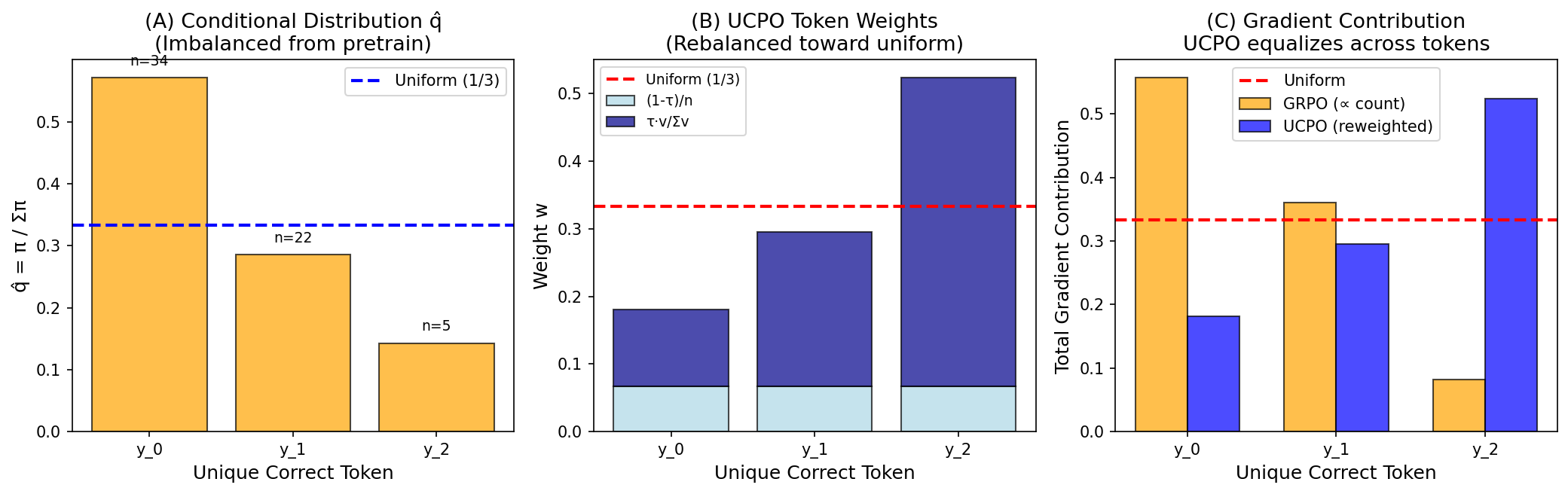}
    \caption{UCPO reweights gradient mass within the correct set.
Starting from an imbalanced conditional distribution over correct tokens, UCPO assigns larger weight to underrepresented correct tokens and smaller weight to dominant ones. Unlike GRPO, whose total gradient contribution is proportional to sample count, UCPO approximately equalizes total contribution across unique correct tokens while preserving the total advantage mass.}
    \label{fig:ucpo_grad_weight}
\end{figure*}

\subsection{UCPO Preserves Diversity via Gradient Redistribution}
\label{app:finding_4_ucpo_weights}
UCPO modifies training dynamics by redistributing gradient signal within the correct set according to the reweighting mechanism derived in Section ~\ref{sec:UCPO}.

As shown in Figure~\ref{fig:method_comp}(a), UCPO maintains a balanced conditional distribution:
$
q_\theta(y \mid \mathcal{Y}^{+}) \approx \text{uniform}.
$
Figure~\ref{fig:method_comp}(d) shows that conditional entropy remains high throughout training, while Figure~\ref{fig:method_comp}(e) shows that correctness mass $Z$ is preserved. 

To isolate the mechanism of UCPO, we further analyze gradient contributions per solution. Figure~\ref{fig:ucpo_grad_weight}(b)  shows that UCPO equalizes gradient contributions across correct solutions and removes the implicit bias toward dominant modes while preserving total gradient mass. In contrast to GRPO where contributions scale with sample count (Figure~\ref{fig:ucpo_grad_weight}(a)), leads to dominance of frequently sampled solutions.

These results demonstrate that UCPO counteracts the sampling-induced imbalance by redistributing updates within the correct set, preventing concentration while maintaining correctness.

\begin{figure*}[!ht]
        \centering
        \includegraphics[width=\linewidth]{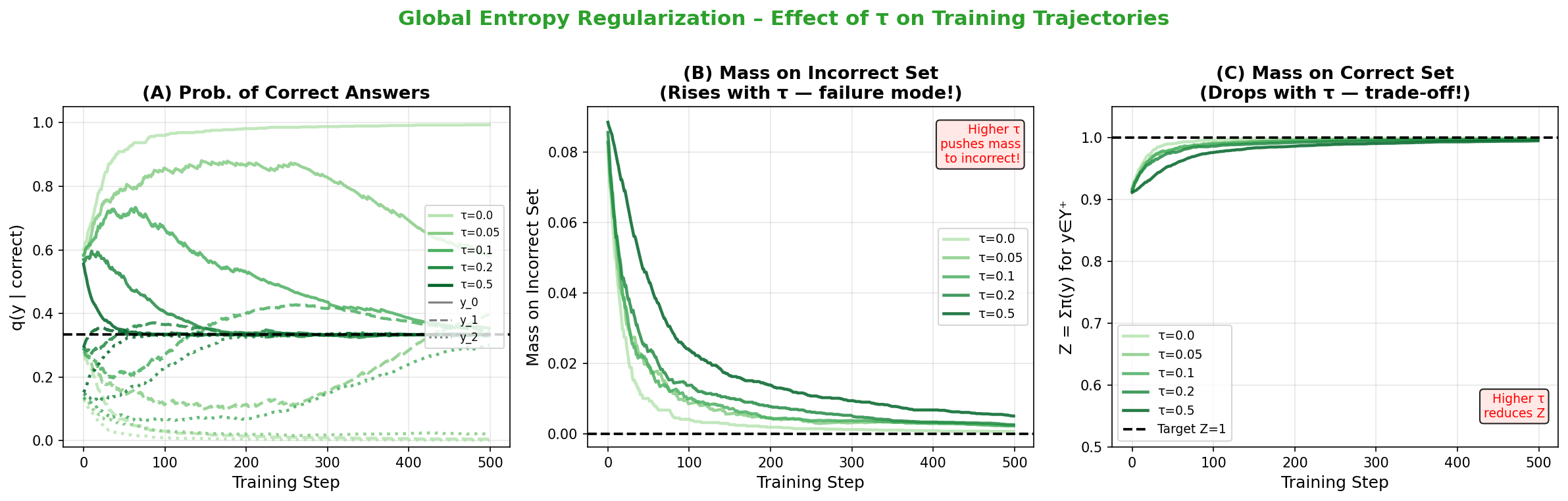}
    \caption{Effect of global entropy regularization strength in the controlled environment. (A) As the entropy coefficient $\tau_\text{ent}$ increases, the conditional distribution over correct tokens remains diverse but at a cost. (B) Mass on incorrect tokens grows with $\tau_\text{ent}$, confirming that the entropy bonus incentivizes incorrect outputs equally.
(C) Correctness mass $Z_\theta$ decreases as $\tau_\text{ent}$ increases, demonstrating the fundamental trade-off: global entropy regularization improves diversity only by sacrificing correctness, as it acts on the entire output space rather than within the correct set.}
    \label{fig:global_ent_reg_mass}
\end{figure*}

\subsection{Global Entropy Regularization Increases Entropy on the Wrong Support}
\label{app:finding_5_global_ent}

We analyze global entropy regularization in Figure~\ref{fig:method_comp} and Figure~\ref{fig:global_ent_reg_mass}.

Figure~\ref{fig:method_comp}(d) shows that entropy regularization increases overall entropy, but Figure~\ref{fig:method_comp}(f) reveals that this increase is partly driven by probability mass assigned to incorrect outputs.

Figure~\ref{fig:global_ent_reg_mass}(b) further shows that as $\tau_\text{ent}$ increases, the equilibrium mass on incorrect tokens grows monotonically. Since the entropy bonus is applied uniformly across all outputs, it provides equal incentive to increase the probability on incorrect and underrepresented correct solutions.

As a result, Figure~\ref{fig:global_ent_reg_mass}(c) shows that improvements in entropy $H(q_\theta)$ come at the cost of reduced correctness mass $Z_\theta$. This confirms that global entropy regularization increases entropy on the wrong support and cannot preserve diversity within the correct set.

\begin{figure*}[!ht]
        \centering
        \includegraphics[width=\linewidth]{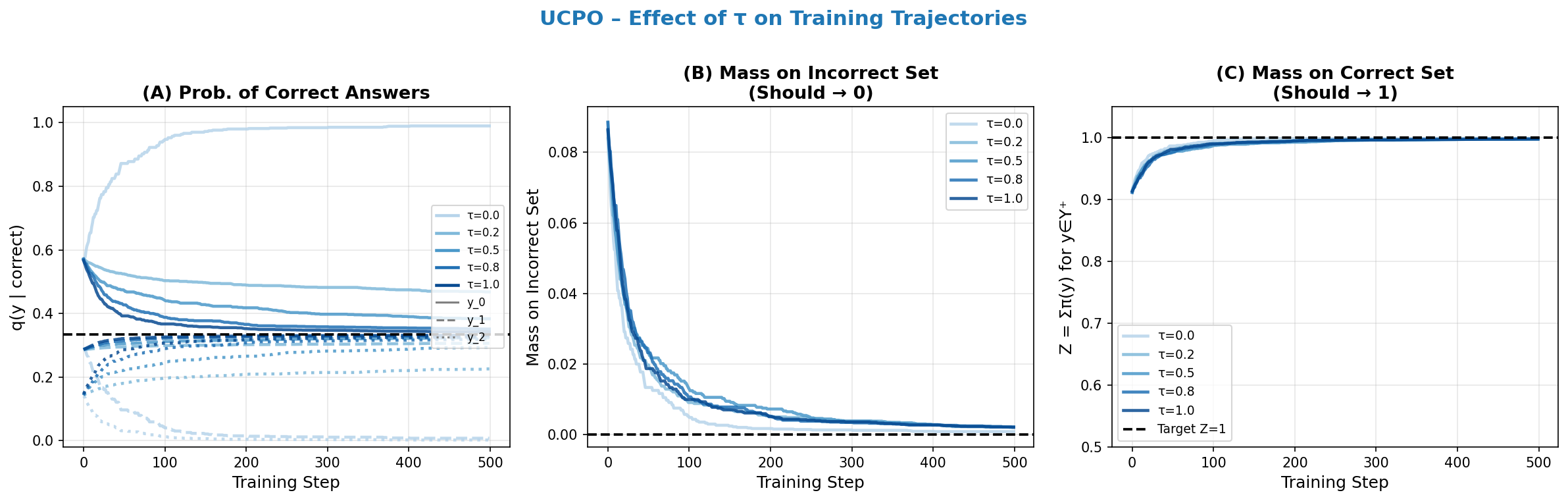}
    \caption{Effect of UCPO interpolation strength $\tau$ in the controlled environment. (A) Increasing $\tau$ progressively maintains diversity within the correct set, preventing concentration onto a single mode. (B) Mass on incorrect tokens remains negligible across all $\tau$ values, confirming that UCPO operates exclusively within the correct set.
(C) Correctness mass $Z_\theta$ stays near 1.0 for all $\tau$, demonstrating that UCPO achieves diversity without the correctness trade-off observed under global entropy regularization.}
    \label{fig:ucpo_mass}
\end{figure*}

\subsection{UCPO Maintains Correctness–Diversity Trade-off}
\label{app:finding_6_ucpo}
Figure~\ref{fig:ucpo_mass} shows that increasing $\tau$ improves diversity while preserving correctness. Unlike entropy regularization, the mass on incorrect outputs remains negligible across all $\tau$, and $Z_\theta$ remains close to 1.0.

This demonstrates that UCPO can increase diversity without incurring the correctness trade-off observed in global entropy methods.

%% file: appendix/why_global_ent_fails.tex
\label{app:global-entropy-fails}
A direct way to operationalize the entropy-regularized criterion from Section~\ref{sec:optimal_policy} is to add a global entropy bonus to the RLVR objective:
\begin{equation}
J_{\mathrm{ent}}(\pi)
=
\mathbb{E}_{y \sim \pi}\!\left[ R(x,y) \right]
+
\tau_\text{ent} \cdot H\!\left(\pi(\cdot \mid x)\right).
\label{eq:global_entropy}
\end{equation}
Although this objective can recover the Uniform-Correct Policy for a suitable choice of $\tau$ (Theorem~\ref{thm:entropy_regularized}), it does not reliably recover UCP in practice. The coefficient $\tau$ is difficult to tune (Table ~\ref{tab:tau_sensitivity_global_entropy}), and for suboptimal values the optimal policy need not remain uniform-correct. Further, the induced optimization dynamics do not directly target balanced allocation within the correct set.

First, the entropy term is computed over the full output space $\mathcal{Y}$, including incorrect solutions, rather than within the correct set $\mathcal{Y}^+(x)$. The reward term pushes mass toward correct outputs, while the entropy term pushes mass toward all outputs indiscriminately. This creates a structural trade-off: diversity is encouraged on the wrong support. Empirically, global entropy regularization increases overall entropy but does not improve diversity within the correct set (Section~\ref{sec:diversity_results}), with much of the entropy increase arising from mass allocated to incorrect outputs (Appendix~\ref{app:finding_5_global_ent}).

Second, under GRPO with binary rewards, when all rollouts for a prompt are correct (or all incorrect), the group-normalized advantages are $A_i = 0$ for every rollout. In this case, the reward term contributes no gradient, and the update is driven entirely by the entropy bonus. The policy is then pushed to spread mass across all outputs, including incorrect ones, rather than preserving diversity within the correct set. This failure mode is particularly common for easier prompts, where the model achieves high accuracy early in training.
Thus, global entropy regularization fails to approximate the optimum of Eq.~\eqref{eq:entropy_regularized_objective}: diversity should be encouraged within correct responses, not across all outputs.

Table~\ref{tab:tau_sensitivity_global_entropy} shows that global entropy regularization provides only marginal and inconsistent gains over GRPO. The best average Pass@64 is obtained at $\tau_{\mathrm{ent}}=0.001$, improving over GRPO by only $+0.19$ points (82.04 vs.\ 81.85). The effect also varies across benchmarks: for example, $\tau_{\mathrm{ent}}=0.0005$ improves AMC Pass@64 but substantially hurts AIME Pass@64, and no coefficient consistently improves all datasets. These results support our claim that global entropy is a brittle mechanism for improving multi-sample coverage, because it acts over the full output space rather than specifically reallocating probability within the correct set.

 \begin{table}[h]
\centering
\caption{Performance at Pass@1 and Pass@64 across benchmarks for different global entropy coefficients for Global Entropy Regularization on DeepSeek-R1-Distill-Qwen-1.5B. All other hyperparameters are held fixed.}
\label{tab:tau_sensitivity_global_entropy}
\resizebox{0.8\textwidth}{!}{
\begin{tabular}{lcccccccccc}
\toprule
& \multicolumn{2}{c}{AIME} 
& \multicolumn{2}{c}{AMC} 
& \multicolumn{2}{c}{MATH} 
& \multicolumn{2}{c}{OlympiadBench} 
& \multicolumn{2}{c}{Avg} \\
\cmidrule(lr){2-3} \cmidrule(lr){4-5} \cmidrule(lr){6-7} \cmidrule(lr){8-9} \cmidrule(lr){10-11}
$\tau_{\text{ent}}$ & @1 & @64 & @1 & @64 & @1 & @64 & @1 & @64 & @1 & @64 \\
\midrule
0.0 (GRPO) 
& 19.06 & 63.33 
& 55.27 & 87.95 
& 71.57 & 97.60 
& 42.63 & 78.52 
& 47.13 & 81.85 \\

0.0005
& 19.79 & 53.33 
& 54.84 & 92.77 
& 69.45 & 97.20 
& 41.69 & 77.93 
& 46.44 & 80.31 \\

0.001
& 19.43 & 63.33 
& 56.12 & 89.16 
& 72.23 & 97.60
& 43.03 & 78.07 
& 47.70 & 82.04\\

0.003
& 19.38 & 60.00 
& 54.71 & 90.36
& 71.45 & 97.20 
& 42.58 & 77.93 
& 47.03 & 81.37 \\

0.005
& 19.48 & 63.33 
& 55.25 & 90.36 
& 72.08 & 97.20 
& 43.86 & 77.04 
& 47.67 & 81.98 \\
\bottomrule
\end{tabular}
}
\end{table}

%% file: appendix/App_exp_details.tex
\label{app:implementation_det}
\subsection{Datasets}
We use the RL training dataset  from \citet{liao2025enhancing}, derived from DeepScaleR 40k corpus \citep{luo2025deepscaler}. The training pool consists of competition-style mathematical reasoning problems drawn from MATH \citep{hendrycksmath2021}, AIME (1983–2023) problem sets, AMC (pre-2023) contests, and Omni-MATH \citep{gao2024omni},  with difficulty rebalancing to reduce trivial problems and improve RL training efficiency. The final training set contains 10,000 problems, with an additional 500 problems reserved for validation.

\subsection{Models}
We evaluate on three models spanning different scales and pretraining approaches: DeepSeek-R1 series: DeepSeek-R1-Distill-Qwen-7B and DeepSeek-R1-Distill-Qwen-1.5B \citep{DeepSeekAI2025DeepSeekR1IR} as well as Qwen2.5-7B \citep{qwen2.5}.
These models span both general model and distillation models, while also covering different parameter scales (1.5B to 7B). All models are initialized from publicly released checkpoints and fine-tuned using the same RLVR training pipeline.

\subsection{Implementation and Evaluation Details}
Hyperparameters are listed in 
Table ~\ref{tab:training_hyperparameters_ds_1.5B}, ~\ref{tab:training_hyperparameters_ds_7B} and ~\ref{tab:training_hyperparameters_qwen_7B}.  For evaluation we report Pass@1 and Pass@K ($K=64$) across five benchmarks. Pass@K measures the fraction of problems for which at least one of $K$ sampled completions is accepted by the verifier. For each benchmark problem, we sample 64 rollouts at temperature 1.0. For inference for each benchmark problem, we sample rollouts using temperature $1.0$ and maximum generation length of $3072$ for DeepSeek-R1-Distill-Qwen-1.5B and DeepSeek-R1-Distill-Qwen-7B and $4096$ for Qwen-2.5-7B. 
\begin{table}[t]
\centering
\small
\setlength{\tabcolsep}{8pt}
\begin{tabular}{l c l c}
\toprule
\textbf{Parameter} & \textbf{Value} & \textbf{Parameter} & \textbf{Value} \\
\midrule
Pretrained Model        & DeepSeek-R1-Distill-Qwen-1.5B   & Device               & 4 $\times$ Nvidia-H100 \\
Prompts per batch       & 64                & Generations per prompt & 8 \\
Gradient update per RL step & 1            & Max prompt length    & 1024 \\
Max response length     & 3072              & Learning rate        & $5 \times 10^{-6}$ \\
Clip ratio low          & 0.2               & Clip ratio high      & 0.2 \\
Epochs         & 2               & $\beta$              & 0.0 \\
Rollout temperature  & 1               & $\tau$           & 0.2 \\

\bottomrule
\end{tabular}
\caption{Training configuration for DeepSeek-R1-Distill-Qwen-1.5B}
\label{tab:training_hyperparameters_ds_1.5B}
\end{table}

\begin{table}[t]
\centering
\small
\setlength{\tabcolsep}{8pt}
\begin{tabular}{l c l c}
\toprule
\textbf{Parameter} & \textbf{Value} & \textbf{Parameter} & \textbf{Value} \\
\midrule
Pretrained Model        & DeepSeek-R1-Distill-Qwen-7B   & Device               & 8 $\times$ Nvidia-H100 \\
Prompts per batch       & 64                & Generations per prompt & 8 \\
Gradient update per RL step & 1            & Max prompt length    & 1024 \\
Max response length     & 3072              & Learning rate        & $2 \times 10^{-6}$ \\
Clip ratio low          & 0.2               & Clip ratio high      & 0.2 \\
Epochs         & 2               & $\beta$              & 0.0 \\
Rollout temperature  & 1                & $\tau$           & 0.2 \\
\bottomrule
\end{tabular}
\caption{Training configuration for DeepSeek-R1-Distill-Qwen-7B}
\label{tab:training_hyperparameters_ds_7B}
\end{table}

\begin{table}[t]
\centering
\small
\setlength{\tabcolsep}{8pt}
\begin{tabular}{l c l c}
\toprule
\textbf{Parameter} & \textbf{Value} & \textbf{Parameter} & \textbf{Value} \\
\midrule
Pretrained Model        & Qwen2.5-Math-7B   & Device               & 8 $\times$ Nvidia-H100 \\
Prompts per batch       & 64                & Generations per prompt & 8 \\
Gradient update per RL step & 1            & Max prompt length    & 1024 \\
Max response length     & 4096              & Learning rate        & $1 \times 10^{-6}$ \\
Clip ratio low          & 0.2               & Clip ratio high      & 0.2 \\
Epochs         & 2               & $\beta$              & 0.0 \\
Rollout temperature  & 1              & $\tau$           & 0.2 \\

\bottomrule
\end{tabular}
\caption{Training configuration for Qwen2.5-Math-7B}
\label{tab:training_hyperparameters_qwen_7B}
\end{table}
 

%% file: appendix/App_additional_results.tex
 \subsection{Pass@K Experiments}
\label{app:more_pass_K_baselines}

\begin{figure*}[!ht]
    \centering
        \includegraphics[width=\linewidth]{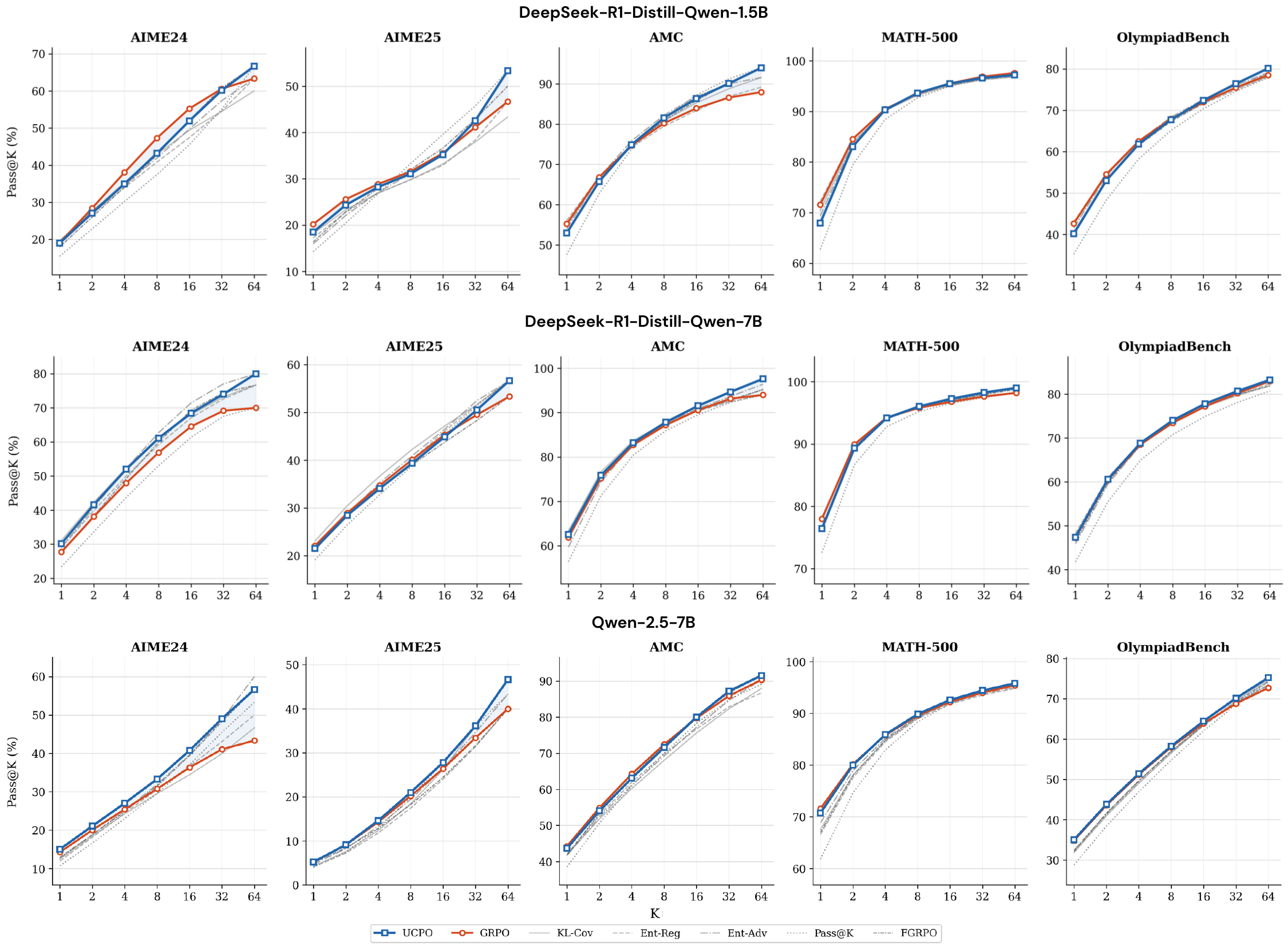}
   \caption{Pass@k curves across mathematical reasoning benchmarks and models. UCPO consistently achieves higher Pass@k, indicating better preservation of diverse correct solution paths. This behavior is consistent with UCPO redistributing gradient signal within the correct set, counteracting the collapse mechanism.}
    \label{fig:passk_all5}
\end{figure*}

Figures~\ref{fig:passk_all5} present the  Pass@K curves for $K \in \{1, 2, 4, 8, 16, 32, 64\}$ on DeepSeek-R1-Distill-Qwen-1.5B,  DeepSeek-R1-Distill-Qwen-7B and Qwen-2.5-7B, respectively, across all five benchmarks (AIME24, AIME25, AMC, MATH, OlympiadBench). We compare UCPO against all baselines: GRPO, KL-Cov, Ent-Reg, Ent-Adv, Pass@K Training, and FGRPO.
  UCPO curves maintain improved performance through $K = 64$. 
The gap between UCPO and GRPO widens monotonically with $K$ especially across harder datasets such as AIME. On the 7B model, this gap is particularly pronounced on AIME24 (up to $+10$ points at $K=64$).

\subsection{Cost Overhead}
\label{app:cost}

\begin{figure}[h]
\centering
\includegraphics[width=0.7\linewidth]{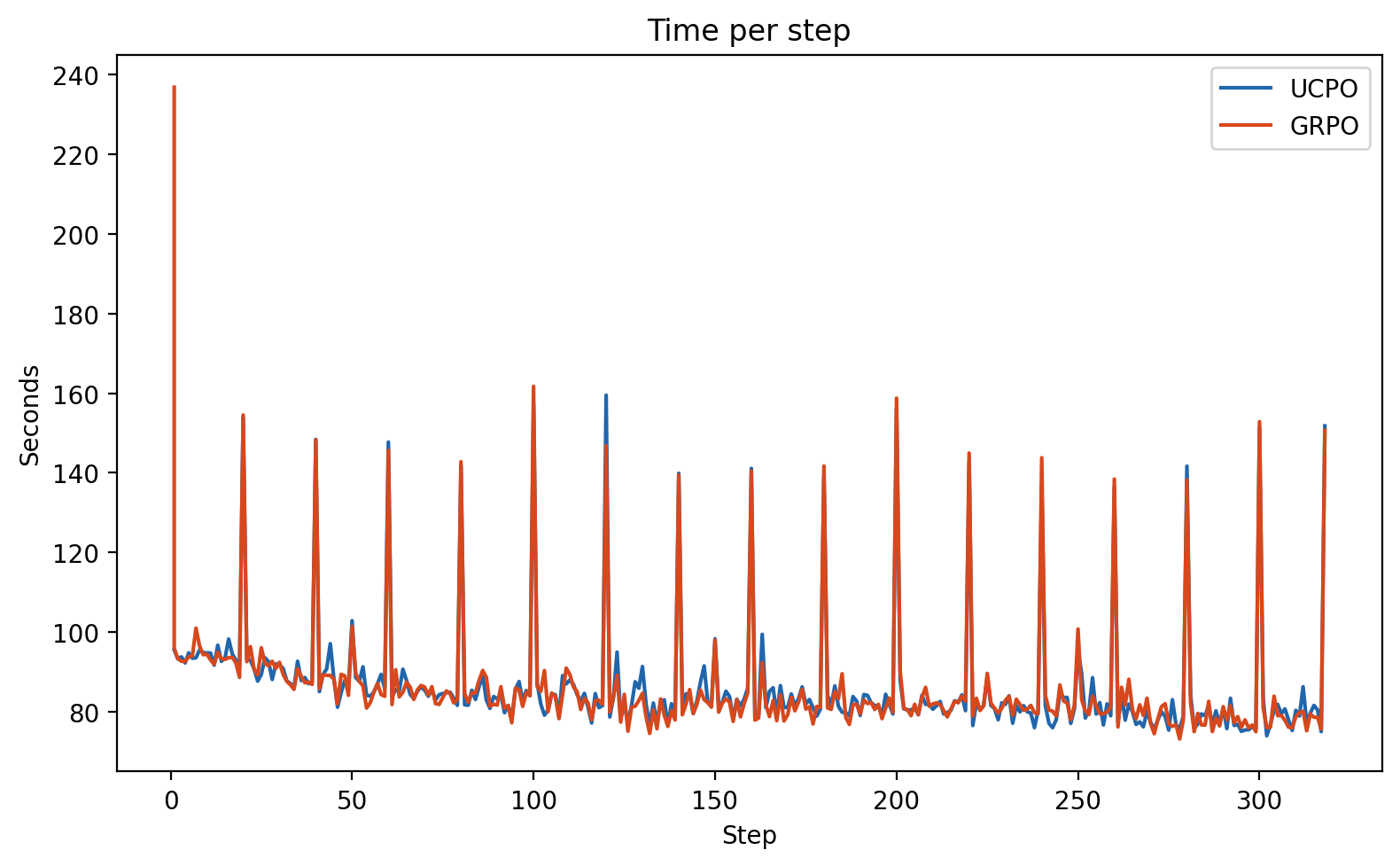}
\caption{Wall-clock time per training step for GRPO and UCPO 
under identical configuration (DeepSeek-R1-Distill-Qwen-1.5B, 
4$\times$H100). The two methods exhibit nearly identical 
training cost throughout. Periodic spikes correspond to 
checkpoint saving.}
\label{fig:cost}
\end{figure}

Figure~\ref{fig:cost} shows that UCPO and GRPO have 
effectively identical per-step training time. UCPO's only 
additional computation is the $O(n)$ scalar advantage 
reweighting (Eq.~\ref{eq:blended-weights}), which is negligible 
relative to rollout generation and the policy gradient 
backward pass. UCPO requires no additional forward passes, 
no reference model, and no extra rollouts.
 
\subsection{Sensitivity to $\tau$}\label{app:tau_sensitivity}
 
UCPO introduces a single hyperparameter $\tau \in [0,1]$ that interpolates between standard GRPO ($\tau = 0$) and fully uniform-correct weighting ($\tau = 1$). We evaluate sensitivity on DeepSeek-R1-Distill-Qwen-1.5B across four values: $\tau \in \{0.2, 0.3, 0.4, 0.5\}$. 
 \begin{table}[h]
\centering
\caption{Performance at Pass@1 and Pass@64 across benchmarks for different UCPO $\tau$ values on DeepSeek-R1-Distill-Qwen-1.5B. All other hyperparameters are held fixed.}
\label{tab:tau_sensitivity}
\resizebox{0.8\textwidth}{!}{
\begin{tabular}{lcccccccccc}
\toprule
& \multicolumn{2}{c}{AIME} 
& \multicolumn{2}{c}{AMC} 
& \multicolumn{2}{c}{MATH} 
& \multicolumn{2}{c}{OlympiadBench} 
& \multicolumn{2}{c}{Avg} \\
\cmidrule(lr){2-3} \cmidrule(lr){4-5} \cmidrule(lr){6-7} \cmidrule(lr){8-9} \cmidrule(lr){10-11}
$\tau$ & @1 & @64 & @1 & @64 & @1 & @64 & @1 & @64 & @1 & @64 \\
\midrule
0.0 (GRPO) 
& 19.06 & 63.33 
& 55.27 & 87.95 
& 71.57 & 97.60 
& 42.63 & 78.52 
& 47.13 & 81.85 \\

0.1
& 19.38 & 66.67 
& 53.65 & 91.57 
& 68.89 & 97.6 
& 41.26 & 77.48 
& 45.79 & 83.33 \\

0.2 
& 19.01 & 66.67 
& 52.97 & 93.98 
& 67.93 & 97.20 
& 40.15 & 80.15 
& 45.02 & 84.50 \\

0.3 
& 19.01 & 66.67 
& 54.10 & 93.98 
& 68.95 & 97.40 
& 40.51 & 79.41 
& 45.64 & 84.36 \\

0.4 
& 18.96 & 63.33 
& 53.11 & 93.98 
& 73.54 & 97.00 
& 43.24 & 76.44 
& 47.21 & 82.69 \\

0.5 
& 18.23 & 56.67 
& 52.03 & 91.57 
& 70.22 & 97.40 
& 41.98 & 77.78 
& 45.61 & 80.85 \\

\bottomrule
\end{tabular}
}
\end{table}

Results are shown in Table~\ref{tab:tau_sensitivity}. Several patterns emerge. First, the diversity gains (Pass@64) are robust over GRPO ($\tau=0$) by $+1.2$ to $+2.7$ points. Second, Pass@1 remains competitive for $\tau \leq 0.4$. Third, at $\tau = 0.5$, Pass@64 begins to degrade, particularly on AIME ($56.67$, below GRPO). This is consistent with the theoretical expectation: at high $\tau$, the gradient is dominated by the uniform term $\mathbb{E}_{y \sim u}[\nabla \log \pi_\theta(y)]$, which may over-correct by distributing gradient signal too aggressively toward low-probability solutions, the importance weights $v_i = 1/\hat{q}_i$
become large and noisy (high variance), which can destabilize training.
 
These results suggest that $\tau$ provides a robust operating range. The diversity--correctness trade-off is mild: Pass@64 improves substantially while Pass@1 fluctuates within normal variance. We use $\tau = 0.2$ as the default throughout our experiments as it yields the strongest Pass@64 improvements. 

We also evaluated global entropy regularization across different entropy coefficients (Table~\ref{tab:tau_sensitivity_global_entropy}). Compared with UCPO, global entropy regularization underperforms at Pass@64 and  is more sensitive to its coefficient. This shows that global entropy is a less effective mechanism for improving multi-sample coverage, consistent with its action over the full output space rather than the conditional distribution over verifier-accepted responses.

\section{Equation Diversity}
We compute equational diversity following \citep{hu2025diversity}. For each prompt, we consider only correct rollouts. From each rollout, we extract
formulas appearing inside LaTeX math delimiters
(\texttt{\$...\$}, \texttt{\textbackslash(...\textbackslash)}, \texttt{\textbackslash[...\textbackslash]})
within the first $L$ characters of the response.
For rollout $i$, let $F_i$ denote the set of extracted formulas.

We define the equation uniqueness score for rollout $i$ as
\[
D_{\mathrm{eq}}^{(i)}
=
\frac{
\left| F_i \setminus \bigcup_{j \neq i} F_j \right|
}{
\max\left(1, |F_i|\right)
}.
\]

This score measures the fraction of formulas in rollout $i$ that do not appear in
any other correct rollout for the same prompt.
We average $D_{\mathrm{eq}}^{(i)}$ across correct rollouts to obtain a
prompt-level score, and then average across prompts within each dataset.
Higher values indicate greater equation-level diversity among correct solutions.

\begin{figure*}[!ht]
    \centering
        \includegraphics[width=\linewidth]{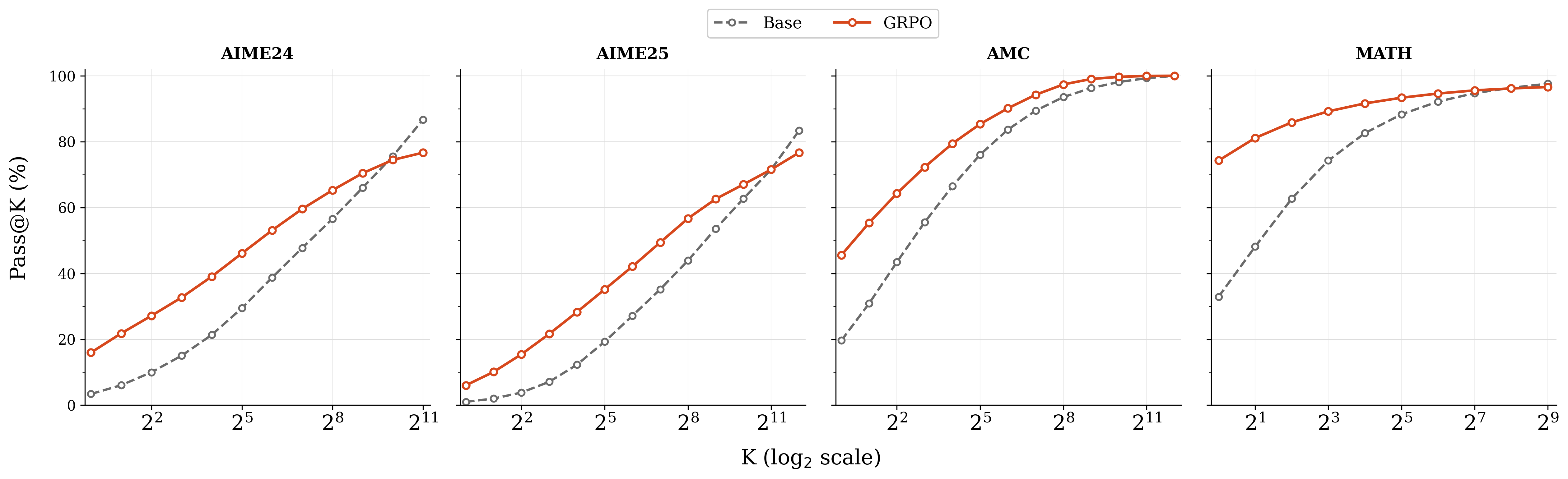}
    \caption{Pass@K comparison between the base model and GRPO on Qwen-2.5-7B. While GRPO improves Pass@1, it underperforms the base model at larger K, indicating contraction of the solution space.}
    \label{fig:div_collpase_base_vs_grpo}
\end{figure*}